%% file: main.tex
\documentclass{article}
\usepackage{kr}

\usepackage{times}
\usepackage{soul}
\usepackage{url}
\usepackage[hidelinks]{hyperref}
\usepackage[utf8]{inputenc}
\usepackage[small]{caption}
\usepackage{graphicx}
\usepackage{amsmath}
\usepackage{amsthm}
\usepackage{booktabs}
\usepackage{algorithm}
\usepackage{algorithmic}
\usepackage{latexsym}

\usepackage{xspace}
\usepackage{amssymb,amsthm,amsmath}
\usepackage{misc} 
\usepackage{color} 
\usepackage[dvipsnames]{xcolor} 
\usepackage{centernot}
\usepackage{array}  
\usepackage{enumitem}
\usepackage[capitalise]{cleveref}
\usepackage{tikz}
\usetikzlibrary{fit,calc}
\usetikzlibrary{decorations.pathmorphing}
\tikzset{zigzag/.style={decorate, decoration=zigzag}}
\definecolor{darkmagenta}{rgb}{0.55, 0.0, 0.55}

\newcommand{\new}[1]{{#1}}

\newcommand{\nb}[1]{}
\newcommand{\comment}[1]{}

\newcommand{\todo}[1]{}

\newcommand{\Aa}{{\cal{A}}}
\newcommand{\Bb}{{\cal{B}}}

\newcommand{\Ff}{{\cal{F}}}
\newcommand{\Ii}{{\cal{I}}}
\newcommand{\Jj}{{\cal{J}}}
\newcommand{\Kk}{{\cal{K}}}

\newcommand{\Mm}{{\cal{M}}}

\newcommand{\Oo}{{\cal{O}}}

\newcommand{\Qq}{{\cal{Q}}}

\newcommand{\Tt}{{\cal{T}}}

\newcommand{\smallplus}{\text{+}} 

\newcommand{\notmodels}{\centernot\models}
\newcommand{\fentails}{\models_{\mathsf{fin}}}
\newcommand{\notfentails}{\notmodels_{\!\!\mathsf{fin}}}

\newcommand{\poly}{\mathrm{poly}}

\renewcommand{\ALCI}{\ensuremath{{\cal{ALC\hspace{-0.06ex}I}}}\xspace}
\renewcommand{\ALCQ}{\ensuremath{{\cal{ALC\hspace{-0.25ex}Q}}}\xspace}
\renewcommand{\ALCO}{\ensuremath{{\cal{ALC\hspace{-0.25ex}O}}}\xspace}
\newcommand{\ALCOI}{\ensuremath{{\cal{ALC\hspace{-0.25ex}O\hspace{-0.06ex}I}}}\xspace}
\newcommand{\ALCOQ}{\ensuremath{{\cal{ALC\hspace{-0.25ex}O\hspace{-0.25ex}Q}}}\xspace}
\newcommand{\ALCOIQ}{\ensuremath{{\cal{ALC\hspace{-0.25ex}O\hspace{-0.06ex}I\hspace{-0.1ex}Q}}}\xspace}

\newcommand{\ALCIQ}{\ensuremath{{\cal{ALCI\hspace{-0.1ex}Q}}}\xspace}

\newcommand{\CQ}{\ensuremath{\mathcal{C\hspace{-0.25ex}Q}}\xspace}
\newcommand{\UCQ}{\ensuremath{\mathcal{U\hspace{-0.15ex}C\hspace{-0.25ex}Q}}\xspace}

\newcommand{\ALCOIt}{\ensuremath{\ALCOI^\smallplus}\xspace}
\newcommand{\ALCIt}{\ensuremath{\ALCI^\smallplus}\xspace}
\newcommand{\ALCOQt}{\ensuremath{\ALCOQ^\smallplus}\xspace}
\newcommand{\ALCQt}{\ensuremath{\ALCQ^\smallplus}\xspace}

\newcommand{\CQt}{\ensuremath{\CQ^\smallplus}\xspace}
\newcommand{\UCQt}{\ensuremath{\UCQ^\smallplus}\xspace}

\DeclareMathSymbol{\shortminus}{\mathbin}{AMSa}{"39}

\newcommand{\ALCIQfwd}{\ensuremath{\ALCIQ^\shortminus}\xspace}
\newcommand{\ALCIQfwdt}{\ensuremath{(\ALCIQfwd)^\smallplus}\xspace}
\newcommand{\ALCOIQfwd}{\ensuremath{\ALCOIQ^\shortminus}\xspace}
\newcommand{\ALCOIQfwdt}{\ensuremath{(\ALCOIQfwd)^\smallplus}\xspace}

\newcommand{\SOI}{\ensuremath{{\cal{S\hspace{-0.2ex}O\hspace{-0.06ex}I}}}\xspace}
\newcommand{\SIF}{\ensuremath{{\cal{S\hspace{-0.06ex}I\hspace{-0.1ex}F}}}\xspace}

\newcommand{\K}{\ensuremath{{\cal{K}}}\xspace}

\renewcommand{\concepts}{\ensuremath{\mn{N_{\mn{C}}}}}
\renewcommand{\roles}{\ensuremath{\mn{N_{\mn{R}}}}}
\newcommand{\individuals}{\ensuremath{\mn{N_{\mn{I}}}}}
\newcommand{\variables}{\ensuremath{\mn{N_{\mn{V}}}}}

\newcommand{\var}{\textit{var}}
\newcommand{\Nom}{\mn{nom}}
\newcommand{\Nomi}{\Nom}
\newcommand{\Ind}{\mn{ind}}
\newcommand{\Rol}{\mn{rol}}

\newcommand{\Tp}{\mn{Tp}}
\newcommand{\CN}{\mn{CN}}
\newcommand{\RCN}{\mn{RCN}}
\newcommand{\tp}{\mn{tp}}
\newcommand{\wtp}{\mn{wtp}}
\newcommand{\twoexp}{\ensuremath{\textsc{2ExpTime}}\xspace}
\newcommand{\ptime}{\ensuremath{\textsc{PTime}}\xspace}

\newcommand{\rel}{\mn{rel}}
\newcommand{\cmp}{\mn{cmp}}

\newcommand{\tup}[1]{\mathbf{#1}}
\newcommand{\rt}{f}
\newcommand{\up}{\tup{g}}
\newcommand{\step}{\mathsf{step}}
\newcommand{\img}{\mathsf{img}}
\newcommand{\Nat}{\mathbb{N}}

\newtheorem{definition}{Definition}
\newtheorem{theorem}{Theorem}
\newtheorem{fact}{Fact}

\newtheorem{lemma}{Lemma}

\title{On Finite Entailment of Non-Local Queries in Description Logics}

\author{
Tomasz Gogacz$^1$ \and
V{\'i}ctor Guti{\'e}rrez-Basulto$^2$ \and
Albert Gutowski$^1$ \and\\
Yazm{\'i}n Ib{\'a}{\~n}ez-Garc{\'i}a$^2$\and
Filip Murlak$^1$ \\
   \affiliations
  $^1$University of Warsaw, Poland \\
  $^2$Cardiff University, UK \\
  \emails
\{t.gogacz, a.gutowski, f.murlak $\!$\}@mimuw.edu.pl, \{gutierrezbasultov, ibanezgarciay $\!\!$\}@cardiff.ac.uk}

\begin{document}

\maketitle

\begin{abstract}
We study the problem of \emph{finite} entailment of ontology-mediated
queries. Going beyond local queries, we allow transitive closure over
roles. We focus on ontologies formulated in the description logics
\ALCOI and \ALCOQ, extended with transitive closure. For both logics,
we show \twoexp upper bounds for finite entailment of unions of
conjunctive queries with transitive closure. We also provide a
matching lower bound by showing that finite entailment of conjunctive
queries with transitive closure in \ALC is \twoexp-hard.
\end{abstract}

\input{introduction}

\input{preliminaries}
\input{plan_of_attack}

\input{nominals}

\input{multiple_roles}

\input{transitive_atoms}

\input{ucqs}

\input{conclusions}

\section*{Acknowledgments}

This work was supported by Poland's National Science Centre
grant 2018/30/E/ST6/00042. 

\bibliography{main,references}
\bibliographystyle{kr}

\appendix
\clearpage
{\huge \bf Appendices}
\input{appendix-nominals}
\clearpage
\input{appendix-multiple_roles}
\clearpage
\input{appendix-transitive_atoms}

\clearpage
\input{hardness}
\end{document}

%% file: introduction.tex
\section{Introduction}

The use of ontologies to provide background knowledge for enriching
answers to queries posed to a database is a major research topic
in the fields of knowledge representation and
reasoning. In this data-centric setting, various options for the
formalisms used to express ontologies and queries exist,
but popular choices are
description logics (DLs)
and either unions of conjunctive queries (UCQs) or navigational queries.
The main reasoning problem in this scenario is query
entailment, which has been extensively investigated for different
combinations of DLs and query languages.  An important assumption in
data-centric applications is that both database instances and the
models they represent are finite.  The study of \emph{finite}
query entailment, where one is interested in reasoning over finite
models only, is thus paramount. Even so, the finite model semantics
has received far less attention than the unrestricted one. 


Prior work on finite query entailment in description logics has
concentrated on \emph{local} queries, UCQs in
particular~\cite{Rosati08,GarciaLS14,AmarilliB15,GogaczIM18,Kieronski18,GogaczGIJM19}.
The single work studying extensions of UCQs \cite{Rudolph16},
brings only undecidability results for navigational queries and
various expressive description logics. In contrast, for unrestricted
entailment a large body of literature is available, providing multiple
positive results for navigational queries and both expressive and lightweight DLs
%
%
%
\cite{CalvaneseEO14,StefanoniMKR14,BienvenuOS15,JungLMS17,GuIbJu-AAAI18,GogaczGIJM19,BednarczykR19}.
The lack of research on finite entailment of non-local queries (e.g.\
variants of regular path queries) comes as a surprise since, as noted
in the above works, they are necessary in data centric applications
dealing with graph databases.

In this paper, we close the distance to the undecidability frontier
for finite entailment delineated by \citeauthor{Rudolph16}~(\citeyear{Rudolph16})
by identifying some decidable classes of non-local ontology-mediated queries.
We focus on UCQs with transitive closure over roles;
from the viewpoint of DL applications, transitive closure is arguably
one of the most useful features of regular expressions over roles. As the
ontology component, we consider extensions of the DL \ALC, allowing for
transitive closure of roles. The study of finite entailment is relevant
for this combination because, unlike for plain CQs, 
query entailment of CQs with transitive closure is not
finitely controllable even for \ALC, and thus finite and unrestricted
entailment do not coincide. As a consequence, dedicated algorithmic
methods and lower bounds need to be developed.


Our main finding is that
finite entailment of UCQs with transitive closure over
\ALCOIt or \ALCOQt knowledge bases is \twoexp-complete.
As unrestricted entailment of positive regular path queries is
\twoexp-complete for both logics~\cite{CalvaneseEO14,BednarczykR19},
and the hardness already holds for CQs and both \ALCI and \ALCO
~\cite{DBLP:conf/cade/Lutz08,DBLP:conf/kr/NgoOS16}, we have that the
complexity of our setting is the same over arbitrary and finite
models.
As finite entailment of two-way regular path queries over
$\mathcal{ALCIOF}$ knowledge bases is undecidable~\cite{Rudolph16},
our results are quite close to the undecidability frontier.
%

While there exist positive results on finite entailment in
expressive DLs related to \ALCOIt or \ALCOQt, they are all about local
queries. For example, the \twoexp algorithms for finite entailment of
UCQs over \SOI, \SIF, and \SOQ knowledge
bases~\cite{GogaczIM18,GogaczGIJM19} are considerably easier to obtain
due to the relatively simple structure of transitive roles. For the
more expressive \SHOIF, the problem becomes
undecidable~\cite{Rudolph16}.  Indeed, to the best of our knowledge,
this paper offers the first positive results on finite entailment of
non-local queries in description logics.


%
%

In order to show our main result, we provide a series of intermediate
reductions allowing us to work with simpler variants of the input
logics and query formalisms. The reductions apply to both logics,
requiring sometimes conceptually different proofs for each of
them. Our proofs encompass different techniques. We use unravelling
operations to establish a tree-like model property, i.e.\ to show that
if a query is not entailed by a knowledge base, then there is a
tree-like counter-model. This in turn serves as the basis for
automata-based approaches to finite entailment. We also use the
coloured blocking principle to construct appropriate finite
counter-models out of infinite tree-like counter-models.

%% file: preliminaries.tex
\section{Preliminaries}\label{sec:preli}

\subsection{Description Logics} 

We consider a vocabulary consisting of countably infinite disjoint
sets of \emph{concept names} $\concepts$, 
\emph{role names} $\roles$, and \emph{individual names} $\mn{N_I}$.
A \emph{role} is a role name  or an \emph{inverse role} $r^-$.
The \emph{(transitive-reflexive) closure} of a role $r$ is
$r^*$. \emph{$\ALCOIQfwdt$-concepts $C,D$} are defined by the grammar 
\[
  C,D ::= A
  \mid \neg C
  \mid C \sqcap D
  \mid \{a\}
  \mid \exists s. C
  \mid \qnrleq n {s'} C  
\]
where $A \in \mn{N_C}$, $s$ is a role or the closure of a role, $n
\geq 0$ is a natural number given in binary, and $s'$ is a role name or
the closure of a role name.  We will use $\qnrgeq n {s'} C$ as
abbreviation for $\neg \qnrleq {n{-}1} {s'} C$, together with standard
abbreviations $\bot$, $\top$, $C\sqcup D$, $\forall s.C$.
Concepts of the form $\qnrleq n  {s'}  C $, $\qnrgeq n {s'} C$,
and $\{a\}$ are called 
\emph{at-most restrictions}, \emph{at-least restrictions},
and \emph{nominals}, respectively. Note that in
$\ALCOIQfwdt$-concepts, \emph{inverse roles} are not allowed in
at-most and at-least restrictions. 

An \emph{$\ALCOIQfwdt$-TBox \Tmc} is a finite set of \emph{concept
inclusions (CIs)} $C\sqsubseteq D$, where $C,D$ are
$\ALCOIQfwdt$-concepts.  An \emph{ABox} $\Amc$ is a finite non-empty
set of \emph{concept} and \emph{role assertions} of the form $A(a)$,
$r(a,b)$, $r^*(a,b)$ where $A \in \mn{N_C}$, $r \in \mn{N_R}$ and
$\{a,b\} \subseteq \mn{N_I}$.  
 A \emph{knowledge base (KB)} is a pair
$\Kk=(\Tt, \Aa)$. We write $\CN(\Kk)$, $\Rol(\Kk)$, $\Nom(\Kk)$, and
$\Ind(\Kk)$ for, respectively, the set of \emph{all concept names,
role names, nominals, and individuals occurring in $\Kk$}. The
\emph{counting threshold} of $\Kk$ is one plus the greatest number
used  in $\Kk$.  We let $\|\Kk\|$ be the total size of the
representation of $\Kk$.

Sublogics of $\ALCOIQfwdt$ are defined by forbidding any subset of the
following features: nominals ($\Oo$), inverse roles ($\Ii$), counting
($\Qq$), and closure of roles ($^\smallplus$); this is indicated by dropping
corresponding letter or decoration from the name of the logic. The symbol
$^\shortminus$ indicates restricted interaction between inverse roles
and counting restrictions,  so if either is forbidden,  $^\shortminus$
is also dropped. The focus of this paper is on the logics $\ALCOIt$ 
and $\ALCOQt$, obtained by forbidding counting and inverse roles,
respectively. Other sublogics of $\ALCOIQfwdt$, in particular
$\ALCIQfwdt$, are used to uniformize and share fragments of
arguments.

\subsection{Interpretations} 

The semantics is given as usual via \emph{interpretations} $\Imc =
(\Delta^\Imc, \cdot^\Imc)$ consisting of a non-empty \emph{domain}
$\Delta^\Imc$ and an \emph{interpretation function $\cdot^\Imc$}
mapping concept names to subsets of the domain and role names to
binary relations over the domain. Further, we adopt the \emph{standard
name assumption}, i.e., $a^\Imc=a$ for all $a \in \mn{N_I}$.  The
interpretation of complex concepts $C$ is defined in the usual
way~\cite{DLBook}. An interpretation $\Imc$ is a
\emph{model of a TBox \Tmc}, written $\Imc\models\Tmc$ if
$C^\Imc\subseteq D^\Imc$ for all CIs $C\sqsubseteq D\in \Tmc$.  It is
a \emph{model of an ABox \Amc}, written $\Imc\models \Amc$, if $a\in
A^\Imc$ for all $A(a)\in \Amc$,  $(a,b)\in r^\Imc$ for all $r(a,b)\in
\Amc$, and $(a,b)\in (r^\Imc)^*$ for all $r^*(a,b)\in \Amc$, where
$ (r^\Imc)^*$ is the usual transitive-reflexive closure of the binary
relation $r^\Imc$. Finally, $\Imc$ is a \emph{model of a KB
$\Kk=(\Tt, \Aa)$}, written $\Imc\models \Kmc$, if $\Ii\models \Tt$ and 
$\Imc\models \Aa$.

An interpretation $\Ii$ is \emph{finite} if $\Delta^\Ii$ is finite. An
interpretation $\Imc'$ is a \emph{sub-interpretation} of $\Imc$,
written as $\Imc'\subseteq \Imc$, if $\Delta^{\Imc'}\subseteq
\Delta^\Imc$, $A^{\Imc'}\subseteq A^\Imc$, and $r^{\Imc'}\subseteq
r^{\Imc}$ for all $A\in\mn{N_C}$ and $r\in \mn{N_R}$. For $\Sigma
\subseteq \concepts \cup \roles$, $\Ii$ is a
\emph{$\Sigma$-interpretation} if $A^\Ii=\emptyset$ and
$r^\Ii=\emptyset$ for all $A\in \concepts \setminus \Sigma$ and $r\in
\roles\setminus\Sigma$. The \emph{restriction of $\Imc$ to signature
$\Sigma$} is the maximal $\Sigma$-interpretation $\Imc'$ with
$\Imc'\subseteq \Imc$. The \emph{restriction of $\Imc$ to domain
$\Delta$}, written $\Ii\upharpoonright \Delta$, is the maximal
sub-interpretation of $\Imc$ with domain $\Delta$. The union $\Ii \cup
\Jj$ of $\Ii$ and $\Jj$ is an interpretation such that $\Delta^{\Ii
\cup \Jj} = \Delta^{\Ii} \cup \Delta^{\Jj}$, $A^{\Ii \cup \Jj} =
A^{\Ii} \cup A^{\Jj}$, and $r^{\Ii \cup \Jj} = r^{\Ii} \cup r^{\Jj}$
for all $A\in\mn{N_C}$ and $r\in \mn{N_R}$.

A \emph{homomorphism} from interpretation $\Ii$ to interpretation
$\Jj$, written as $h : \Ii \to \Jj$ is a function $h : \Delta^\Ii \to
\Delta^\Jj$ that preserves roles, concepts, and individual names: that
is, for all $r \in \roles$, $(h(d), h(e)) \in r^\Jj$ whenever $(d,e)
\in r^\Ii$, for all $A \in \concepts$, $h(d) \in A^\Jj$ whenever $d
\in A^\Ii$, and $h(a)= a$ for all $a \in \Ind(\Kmc)$.

\subsection{Queries and Finite Entailment}

Let $\variables$ be a countably infinite set of \emph{variables}. An
\emph{atom} is an expression of the form $A(t)$, $t = t'$, $r(t,t')$,
or $r^*(t,t')$ with $A \in\concepts$, $r \in \roles$, and $t,t' \in
\variables\cup\individuals$, referred to as \emph{concept},
\emph{equality}, \emph{role}, and \emph{transitive atoms},
respectively.
A \emph{conjunctive query (with transitive atoms)} is a set of atoms,
understood as the conjunction thereof.  We write
$\CQt$ for conjunctive queries (with transitive atoms), and $\CQ$ for
conjunctive queries without transitive atoms.
Let $\var(q)$ be the set of variables occurring in the atoms of $q \in \CQt$.
A \emph{match for $q$ in $\Imc$} is a function $\eta:\var(q)\to
\Delta^\Ii$ such that $\Ii,\eta\models q$ under the standard semantics
of first-order logic, assuming that the extension of $r^*$ is the
reflexive-transitive closure of the extension of $r$. 
An interpretation $\Ii$ \emph{satisfies} $q$, written $\Ii\models q$, if
there exists a match for $q$ in $\Ii$.

Fix $p,q \in \CQt$ and a function $\eta:\var(p)
\to \variables \cup \individuals$. Let $\eta(p)$ be 
obtained from $p$ by substituting each $x\in\var(p)$ with $\eta(x)$.
We call $\eta$ a \emph{homomorphism} if $\eta(p) \subseteq q^*$, where
$q^*$ is obtained by saturating $q$ as follows for all $r$ and all $s,s' \in \{r, r^*\}$:
if $\{x\,{=}\,y, y\,{=}\,z\} \subseteq q$, add $x\,{=}\,z$,
if $\{s(x,y), y\,{=}\,y'\} \subseteq q$, add $s(x,y')$,
if $\{s(x,y), x\,{=}\,x'\} \subseteq q$, add $s(x',y)$,
if $\{s(x, y), s'(y,z)\} \subseteq q$, add $r^*(x,z)$.
Let $\Ii\models q$. If there is a homomorphism from $p$ to $q$,
then $\Ii\models p$. If there is a homomorphism from $\Ii$ to $\Jj$,
then $\Jj\models q$.

A union of conjunctive queries is a finite set of $\CQt$s. We write
$\UCQt$ for unions of conjunctive queries, and $\UCQ$ for unions of
conjunctive queries without transitive atoms.
An interpretation $\Ii$ satisfies $Q\in\UCQt$, written as $\Ii\models
Q$, if $\Ii \models q$ for some $q \in Q$.
%
%
A \emph{fragment} of $Q$ is a $\CQt$ formed by a \new{connected}
subset of atoms of some $q \in Q$. We let $\|Q\|$ denote the number of all
fragments of $Q$; note that it is exponential in $\max_{q\in Q} |q|$.


We say that $\Kk$ \emph{finitely entails} $\Qq$, written
$\Kk\fentails Q$, if each finite model of $\Kk$ satisfies $Q$. A model
of $\Kk$ that does not satisfy $Q$ is a \emph{counter-model}.  The
\emph{finite entailment problem} asks if a given KB $\Kk$ finitely
entails a given query $Q$.


We also consider \emph{finite entailment modulo types}, which allows
more precise complexity bounds.  A \emph{unary $\Kk$-type} is a subset
of $\CN(\Kk)$ including either $A$ or 
$\bar A$ for each $A \in \CN(\Kk)$. Let $\Tp(\Kk)$ be the set of all
unary $\Kk$-types. For an interpretation $\Ii$ and an element $d \in
\Delta^\Ii$, the \emph{unary $\Kk$-type of $d$ in $\Ii$} is
$\tp^\Ii(d) = \left \{ A \in \CN(\Kk) \bigm | d \in A^\Ii\right\}$. We
say that $\Ii$ \emph{realizes} a unary $\Kk$-type $\tau$ if $\tau =
\tp^\Ii(d)$ for some $d \in \Delta^\Ii$.  For a KB $\Kk$, a query $Q$,
and a set of unary types $\Theta \subseteq \Tp(\Kk)$ we write $\Kk
\fentails^\Theta Q$ if for each finite interpretation $\Ii$ that only
realizes types from $\Theta$, if $\Ii\models \Kk$ then $\Ii\models Q$.
In this context, a counter-model is a model of $\Kk$ that only
realizes types from $\Theta$ and does not satisfy $Q$. 

\subsection{Normal Form and Additional Assumptions}

Without loss of generality, we assume throughout the paper that all
CIs are in one of the following \emph{normal forms}:
%
%
\begin{align*}
  \bigsqcap_i A_i \sqsubseteq \bigsqcup_j B_j,
  \quad A \sqsubseteq \qnrleq n r B,
  \quad A \sqsubseteq \qnrgeq n r B,
  \\
  \quad A \sqsubseteq \qnrleq n {r^*} B,
  \quad A \sqsubseteq \qnrgeq n {r^*} B,
  \\ 
  \quad A \sqsubseteq \forall r^-. B,
  \quad A\sqsubseteq \exists r^-.B,
  \quad A\sqsubseteq \exists (r^-)^*.B,
\end{align*}
where $A,A_i,B,B_j$ are concept names or nominals, $r\in \mn{N_R}$,
and empty disjunction and conjunction are equivalent to $\bot$ and
$\top$, respectively. In logics without counting, the number $n$
in at-most restrictions must be 0, and in at-least restrictions it
must be 1.  
%
%
We also assume that for each concept name $A$ used in $\Kk$ there is a
\emph{complementary} concept name $\bar A$ axiomatized with CIs $\top
\sqsubseteq A \sqcup \bar A$ and $A \sqcap \bar A \sqsubseteq \bot$.


A concept name $B\in\CN(\Kk)$ is \emph{relevant} if $\Kk$ contains a
CI of the form $A \sqsubseteq \qnrleq n {r^*} B$ with $n>0$. We let
$\RCN(\Kk)$ denote the set of relevant concept names in $\Kk$.  A
concept name $B$ is \emph{relevant in $\Ii$ for $d\in\Delta^\Ii$ with
respect to $r\in\roles$} if $d \in A^\Ii$ for some CI $A \sqsubseteq
\qnrleq n {r^*} B$ in $\Kk$ \new{with $n>0$}. We call $\Kk$
\emph{sticky} if for each model $\Ii$ of $\Kk$, each $r \in \roles$,
and each $d\in\Delta^\Ii$, all concept names relevant for $d$ with
respect to $r$ are also relevant (with respect to $r$) for each
$r$-successor of $d$. Stickiness of $\Kk$ can be assumed without loss
of generality:
\new{
for each $B\in\RCN(\Kk)$ and $r
\in \Rol(\Kk)$ introduce fresh concept names $B_r$ and $\overline B_r$
axiomatized with $\top \sqsubseteq B_r \sqcup \overline B_r $, $B_r
\sqcap \overline B_r \sqsubseteq \bot$, $B_r \sqsubseteq \qnrleq 0 r
{\overline B_r}$, $B_r \sqsubseteq \qnrleq N {r^*} B$, where $N$ is
the counting threshold in $\Kk$, and add $A \sqsubseteq B_r$ for each
CI $A \sqsubseteq \qnrleq n {r^*} B$ in $\Kk$.
}

A variable $y$ is \emph{linking} in $q\in\CQt$ if the only atoms in
$q$ using $y$ are $r^*(x, y), r^*(y, z)$ for some $x, z \in
\var(q)$. In a match $\eta$ for $q$, $\eta(y)$ can be any node on a
path from $\eta(x)$ to $\eta(z)$. We call $q$ \emph{normalized} if in
every atom over $r^*$ at least one variable is linking, and for every
two atoms $r^*(x, y), r^*(y, z)$ with $y$ linking, exactly one of the
variables $x$ and $z$ is linking too. Each query can be normalized by
first eliminating all linking variables, and then subdividing each
$r^*$ atom into three $r^*$ atoms using two fresh linking variables.
Without loss of generality we can assume that the input
\new{$\UCQt$s consist of normalized connected $\CQt$s}.

%% file: plan_of_attack.tex
\section{Plan of Attack}

Our main technical contribution are the following results.

\begin{theorem}\label{thm:LoweBound}
 Finite entailment of $\UCQt\!$s over \ALCOIt or \ALCOQt knowledge bases is \twoexp-complete. 
 \end{theorem}
%
To prove upper bounds, we will show a series of reductions that will
allow us to deal at the end with a base case where   the query is a
plain UCQ and the knowledge base is `single-role ABox-trivial'  either
in $\ALCIt$ or in $\ALCQt$ without at-most restrictions over closures
of roles. These reductions can be seen as different phases in  our
decision procedure, each of them depending on the previous one. 


\medskip \noindent {\bf I.}
We start by showing in Section~\ref{sec:ABoxesNomMain} that one can eliminate nominals from the input knowledge bases, i.e.\ roughly,  that one can reduce finite entailment of $\UCQt\!$s over  \ALCOIt or \ALCOQt KBs to \ALCIt or \ALCQt KBs, respectively. We further show that the input ABox can be assumed to be `trivial' in the sense that it contains no role assertions and that only one individual name occurs in it.  

\smallskip \noindent {\bf II.}
In Section~\ref{sec:mulroles} we show that the variants obtained in Step I above can be further reduced to the case where the knowledge base contains a single role name. This is in line with seeing \ALCQt and \ALCIt as a fusion of logics~\cite{BaaderLSW02}, where the interaction between different roles is limited. In a nutshell, we show
that there exists a finite counter-model if and only if there exists a tree-like counter-model accepting a decomposition into components interpreting a single role name.  
 

\smallskip \noindent {\bf III.}
In Section~\ref{sec:elim-trans}, assuming  single-role ABox-trivial $\ALCIt$ or 
$\ALCQt$  KBs, we eliminate transitive closure from queries and from at-most
restrictions. This  step is the one requiring  the most technical
effort. We will develop the notion of hybrid decompositions, which are
tree decompositions that associate with each node arbitrarily large
interpretations, but have a certain more subtle parameter bounded. 
We will observe special characteristics of these interpretations that
allow us to establish structural restrictions between neighbouring nodes
in the decomposition.  Roughly, the key result will be  that finite
counter-models can be unravelled  into  counter-models  admitting  a
hybrid decomposition with the above features. This tree-like model
property will then be the basis for automata-based approaches to
finite entailment in this setting.  

\smallskip \noindent {\bf IV.}
 Section~\ref{sec:basecase}  provides a decision procedure for the base case described above.  We will  rely again on tree unravelling and automata-based techniques. 

%

\smallskip
For the matching lower bound, we show the following.

\begin{theorem}
 Finite entailment of  $\CQt\!$s over \ALC knowledge bases is \twoexp-hard. 
 \end{theorem}
The proof is  by reducing the word problem of exponentially space
bounded alternating Turing machines,  which is known to be
\twoexp-hard~\cite{chandraAlternation1981}. The encoding is a direct
adaptation of the one used to show that (arbitrary) entailment of CQs
over $\mathcal{SH}$ knowledge bases is \twoexp-hard~\cite{EiterLOS09},
but use $\CQt\!$s to compensate for the weaker logic.


%% file: nominals.tex
\section{Eliminating Nominals, Trivializing ABoxes}
\label{sec:ABoxesNomMain}

Finite entailment of a $\UCQt$ over an $\ALCOIt$ or $\ALCOQt$ KB can
be reduced to multiple instances of finite entailment of $\UCQt$s over
$\ALCIt$ or $\ALCQt$ KBs, respectively, with \emph{trivial} ABoxes;
the latter means that ABoxes mention only one individual (with a
fully specified unary type) and contain no binary assertions.
Concluding complexity bounds for $\ALCOIt$ or $\ALCOQt$ from bounds
for $\ALCIt$ or $\ALCQt$ requires precise estimations of the number of
these instances and their parameters. The key parameters are the size
of the KB, the number of $\CQt$s in the query, and their size. 

Consider an $\ALCOIt$ or $\ALCOQt$ KB $\Kk$ and $Q \in \UCQt$ with
$\CQt$s of size at most $m$. Then, the reduction gives an at
most doubly exponential number of instances, with KBs of size at most
%
%
$\poly(\|\Kk\|, N^{1+|\RCN(\Kk)|})$ and queries with at most
%
%
$|Q| \cdot \big(\|\Kk\|\cdot N^{|\RCN(\Kk)|}\big)^{O(m)}$
$\CQt$s of size $O(m)$. Importantly, the reduction preserves the counting
threshold $N$ and the set $\RCN(\Kk)$ of relevant concept names.


Let us sketch the argument. Using routine one-step unravelling one can
show that for $\ALCOIt$ it suffices to search for counter-models $\Ii$
that can be decomposed into domain-disjoint interpretations
$\Ii\upharpoonright \Ind(\Kk)$ and $\Ii_1, \dots, \Ii_n$ for some $n
\leq |\Kk| \cdot |\Ind(\Kk)|$ with two kinds of additional edges:
arbitrary edges connecting elements from $\bigcup_i\Delta^{\Ii_i}$
with elements from $\Nom(\Kk)$, and a single edge connecting a
distinguished element $d_i\in\Delta^{\Ii_i}$ with a corresponding
element $e_i \in \Ind(\Kk)$ for each $i\leq n$. Iterating over
possible $\Ii\upharpoonright \Ind(\Kk)$, we check if there exist
$\Ii_i$ that \new{are compatible with}
$\Kk$, provide witnesses for elements from
$\Ind(\Kk)$, and avoid satisfying the query. The properties required
for a single $\Ii_i$ can be encoded as a finite entailment problem for
a modified KB whose ABox describes the unary type of $d_i$, and a
union of selected fragments of $Q$ with some variables
substituted by elements of $\{d_i\} \cup \Nom(\Kk)$. By representing
the existence of an $r$-edge to $a\in \Nom(\Kk)$ with a fresh concept
name $A_{\exists r.\{a\}}$, suitably axiomatized, we can eliminate
nominals from the KB and from the query.

For $\ALCOQt$, unravelling is made harder by at-most restrictions over
closures of roles. We shall treat in a special way all successors of
elements from $\Ind(\Kk)$ that are affected by such at-most
restrictions. Let us call an $r$-successor $e$ of $d \in \Delta^\Ii$
\emph{directly $r$-relevant} if $e\in A^\Ii$ for some $A$ relevant for
$d$ with respect to $r$, and let $\rel_r^\Ii(d)$ be the set of all \emph{$r$-relevant
successors} of $d$; that is, the least set containing all directly
$r$-relevant successors of $d$ and closed under directly $r$-relevant
successors. The key insight is that this set has bounded size. 

%

\begin{fact}[\citeauthor{GuIbJu-AAAI18} \citeyear{GuIbJu-AAAI18}]  \label{fact:relevant}
  If $N$ is the counting threshold in $\Kk$ and
  $\Ii \models \Kk$, then   $|\rel_r^\Ii(d)|\leq
  N^{|\RCN(\Kk)|}$ for all $d\in\Delta^\Ii$ and  $r \in \Rol(\Kk)$.   
\end{fact}

\noindent In the argument for $\ALCOIt$, we replace $\Ind(\Kk)$ with
\[\Ind(\Kk) \cup \bigcup \big \{ \rel^\Ii_r(a) \bigm| {r \in
      \Rol(\Kk)}, {a \in \Ind(\Kk)}\big\}\,,\]
and similarly for $\Nom(\Kk)$.
Appendix A has full details.

In the following sections we shall focus on the decision
procedures for $\ALCIt$ and $\ALCQt$, but to ensure that they give
$\twoexp$ upper bounds for $\ALCOIt$ and $\ALCOQt$ when combined
with the reductions above, we will need more careful complexity analysis.


%% file: multiple_roles.tex
\section{Eliminating Multiple Roles}~\label{sec:mulroles}
Moving on with simplifications, we reduce finite entailment of $\UCQt$s
over $\ALCIt$ or $\ALCQt$ KBs to the single-role case. More 
precisely, we show how to solve one instance of finite entailment by
solving many instances of finite entailment modulo types with
single-role KBs.
 
Let $\Kk = (\Tt,\Aa)$ be an $\ALCIt$ or $\ALCQt$ KB, $Q \in \UCQt$
and $ m = \max_{q\in Q} |q|$. One can decide if 
$\Kk\fentails Q$ in time $O\big(\poly\big(\|\Kk\|, 2^{|\CN(\Kk)|\cdot |Q|
  \cdot 4^{m}}\big)\big)$ using
oracle calls to finite entailment modulo
types with single-role KBs of size at most $\poly(\|\Kk\|,\|Q\|)$,
unions of at most $|Q| \cdot 2^{\poly(m)}$ $\CQt$s of size $\Oo(m)$,
and type sets of size at most $2^{|\CN(\Kk)|\cdot |Q| \cdot 4^{m}}$;
the KBs inherit their counting threshold and 
set of relevant concept names from $\Kk$.

The argument again relies on unravelling, which allows focusing on
\emph{tree-like} counter-models; that is, ones that can be decomposed
into multiple finite $\Sigma_r$-subinterpretations with
$\Sigma_r=\{r\}\cup\CN(\Kk)$ and $r \in \Rol(\Kk)$, called
\emph{bags}, arranged into a (possibly infinite) tree such that:
(1) two bags share a single element if they are neighbours, and are
disjoint if they are not neighbours; (2) each element $d$ occurs in
exactly one $\Sigma_r$-bag for each $r\in \Rol(\Kk)$ and \new{some bag
containing $d$ is the parent of all other bags containing $d$}; (3)~the root bag satisfies
$\Aa$. A tree-like interpretation is a model of $\Kk$ iff each
$\Sigma_r$-bag is a model of the TBox $\Tt_r$, collecting all CIs from $\Tt$ that mention 
only role name $r$.

Evaluating $Q$ over a tree-like interpretation $\Ii$ can be
distributed over bags by means of \emph{$Q$-labellings}, which assign
to each element its \emph{(unary) $Q$-type}, summarizing information
relevant for $Q$ as a set of pairs $(p,V)$ where $p$ is a  fragment of
$Q$ and $\new{\emptyset \neq \,} V 
\subseteq \var(p)$. We shall think of a $Q$-labelling as an extension
$\Ii'$ of $\Ii$ to fresh concept names $A_{p, V}$.
\new{We are interested in $Q$-labellings $\Ii'$ such that $A_{q,
    V}^{\Ii'} = \emptyset$ for all $q \in
  Q$ and $\emptyset \neq V \subseteq \var(q)$;
  we call them \emph{$Q$-refutations}.} 
We call $\Ii'$ \emph{correct} if $e \in A_{p, V}^{\Ii'}$ iff $\eta(V)
= \{e\}$ for some match $\eta$ for $p$ in $\Ii'$. We call $\Ii'$
\emph{consistent} if each bag satisfies the following: for each
partition of a fragment $p'$ into fragments $p, p_1, p_2, \dots, p_k$
with $\var(p_i)\cap \var(p_j) = \emptyset$ for $i\neq j$, $V_i =
\var(p_i) \cap \var(p)$, and $\new{\emptyset \neq} \;V\subseteq
\var(p)$, there is no match $\eta$ for $p$ in the bag such that
$\eta(V_i) = \{e_i\} \subseteq A_{p_i, V_i}^{\Ii'}$ for all $i$ but
$\eta(V) = \{e\} \not\subseteq A_{p', V}^{\Ii'}$.

\begin{lemma} \label{lem:multirolesconsistent}
  \new{The correct $Q$-labelling of $\Ii$ is a
  $Q$-refutation iff $\Ii$ admits a consistent $Q$-refutation.}
\end{lemma}

It remains to show how to find consistent \new{$Q$-refutations $\Jj$ such
that $\Jj \models \Kk$} and prove that they can be turned into finite
counter-models.

\new{Consistent $Q$-refutations} can be recognized by a tree
automaton but there are two obstacles. First, bags are finite, but arbitrarily
large, so tree-like interpretations do not naturally encode as
finitely-labelled trees. Second, the information that needs to be
passed between bags is the unary type of the shared element, and the
unary type of the element shared with the parent bag needs to be
related with the \emph{set of unary types} of elements shared with
child bags. As the number of such types is $2^{\CN(\Kk) \cdot |Q|\cdot
4^{m}}$, this suggests a triply exponential construction. Our response
is to generalize automata to trees over infinite alphabets and avoid
representing the transition relation explicitly.


\begin{definition}
A \emph{(B\"uchi) automaton} $\Bb= (\Sigma, \Gamma,S, I, F, \delta)$
consists of a \emph{node alphabet} $\Sigma$, an \emph{edge alphabet}
$\Gamma$, a finite set $S$ of \emph{states}, sets $I, F\subseteq S$ of
\emph{initial} and \emph{accepting} states, and a \emph{transition
function} $\delta: S \times \Sigma \to 2^{(\Gamma \times S)^*}$
mapping state-letter pairs to sets of words over the alphabet
$\Gamma\times S$.  A \emph{run} $\rho$ of $\Bb$ on a tree $T$ maps
nodes of $T$ to $S$ such that $\rho(\varepsilon) \in I$ and for each
node $v$ with children $v_1, v_2, \dots, v_n$,
\[ (\gamma_1, \rho(v_1)) (\gamma_2, \rho(v_2)) \dots (\gamma_n,
  \rho(v_n))  \in  \delta (\rho(v), \sigma) \,,\]
where $\gamma_i$ is the label of the edge $(v, v_i)$ and $\sigma$ is
the label of the node $v$. A tree $T$ is \emph{accepted} by $\Bb$ if
for some run $\rho$, each branch of $T$ contains infinitely many
nodes from $\rho^{-1}(F)$. All trees accepted by
$\Bb$ form the set \emph{recognized} by $\Bb$.
\end{definition}

Let us build an automaton $\Bb_{\Kk, Q}$ recognizing \new{consistent 
$Q$-refutations}. Because the state space is finite, the automaton
cannot compute which elements are shared between bags: it must be
given this information. We provide it by marking in each bag the
element shared with the parent bag and displaying the shared element
on the edge to the parent.  With that, the actual elements used in the
bags do not matter any more: the element on the edge between bags
indicates the element of the parent bag that should be identified with
the marked element of the child bag. Let $\Delta$ be a countably infinite
set. The node alphabet is the set of finite
$\Sigma_r$-interpretations for $r\in\Rol(\Kk)$ with domains contained
in $\Delta$, and the edge alphabet is $\Delta$. States consist of
$r\in\Rol(\Kk)$ and a unary type $\tau$ that includes no $A_{q, V}$
with $q \in Q$.  A state $(r, \tau)$ is initial if $\tau \supseteq \{A
\bigm | A(a) \in \Aa\}$ where $\{a\} = \Ind(\Kk)$.  All states are
accepting. The transition function ensures that the input tree
represents a tree-like interpretation and that each $\Sigma_r$-bag is
a model of $\Tt_r$ and satisfies the consistency condition.




Our automata are infinite objects, but in the oracle model their
emptiness can be tested in $\ptime$ (see Appendix B).

\begin{fact} \label{fact:emptiness}
  There is a polynomial-time algorithm with an oracle that, for each
 automaton $\Bb = (\Sigma, \Gamma, S, I, F, \delta)$, accepts input
 $(S, I, F)$ with oracle $\step_\Bb$ iff $\Bb$ accepts some tree, where
 $\step_\Bb$ is the set of pairs $(P, q) \in 2^S\times S$ such that
 $\delta(q, \sigma) \cap (P\times \Gamma)^* \neq \emptyset$  for
 some $\sigma \in \Sigma$.
\end{fact}

Deciding if $\big(P, (r, \tau)\big) \in \step_{\Bb_{\Kk,Q}}$ reduces
to a single instance of finite entailment modulo types, for
$\Kk'=(\Tt_r, \Aa_\tau)$ where $\Aa_\tau$ encodes $\tau$ as a set of
assertions on some distinguished individual, $Q'$ obtained, informally
speaking, by taking the union of all matches forbidden in the
consistency condition, and $\Theta'=\big \{ \tau \bigm | \forall \, 
r\in\Rol(\Kk)\; (r, \tau) \in P \cup (r, \tau)\big\}$. Hence, applied
to $\Bb_{\Kk, Q}$, Fact~\ref{fact:emptiness} yields the desired
bounds.

\smallskip
To get finite counter-models we use coloured blocking.
For $d \in\Delta^\Ii$, the \emph{$n$-neighbourhood $N_n^{\Ii}(d)$ of
$d$} is the interpretation obtained by restricting $\Ii$ to elements
$e \in \Delta^\Ii$ within distance $n$ from $d$ in $\Ii$, enriched
with a fresh concept interpreted as $\{d\}$.
A \emph{colouring of $\Ii$ with $k$ colours} is an
extension $\Ii'$ of $\Ii$ to $k$ fresh concept names $B_1, \dots, B_k$
such that $B_1^{\Ii'}, \dots, B_k^{\Ii'}$ is a partition of
$\Delta^{\Ii'} = \Delta^{\Ii}$. We say that    
$d \in B_i^{\Ii'}$ has colour $B_i$. We call $\Ii'$ 
\emph{$n$-proper} if for each $d \in \Delta^{\Ii'}$  
all elements of $N_n^{\Ii'}(d)$ have different colours.

\begin{fact}[\citeauthor{GogaczIM18} \citeyear{GogaczIM18}]
  \label{fact:coloured-blocking}
If $\Ii$ has bounded degree, then for all $n\geq 0$ there exists an
$n$-proper colouring $\Ii'$ of $\Ii$ with finitely many colours.  
Consider interpretation $\Jj$ obtained from $\Ii'$ by redirecting some
edges such that the old target and the new target have
isomorphic $n$-neighbourhoods in $\Ii'$. Then, for each $q \in \CQ$
with at most $\sqrt{n}$ binary atoms, if $\Ii\notmodels q$, then
$\Jj\notmodels q$.  
\end{fact}

From the proof of Fact~\ref{fact:emptiness} it also follows that if
$\Bb$ accepts some tree, then it also accepts a \emph{regular} tree;
that is, a tree that has only finitely many non-isomorphic
subtrees. Each regular tree has bounded branching and uses only
finitely many different labels.  Consequently, the tree-like
counter-model $\Ii$ corresponding to a regular tree accepted by
$\Bb_{\Kk, Q}$ uses finitely many different bags (up to isomorphism),
which means that their size is bounded by some $k$. Therefore,
$\Ii$ has bounded degree and, because $\Sigma_r$-bags are disjoint,
the length of simple directed $r$-paths in $\Ii$ is also bounded by
$k$.

For $\ell = 2k$, let  $Q^{(\ell)}$ be obtained from $Q$ by replacing
each transitive atom $r^*(x, y)$ with the disjunction 
\[ x=y \lor r(x,y) \lor \bigvee_{1<i\leq \ell} r^i(x, y)\, , \] where
$r^i(x, y)$ expresses the existence of an $r$-path of length $i$ from
$x$ to $y$ as $r(x, z_1) \land r(z_1, z_2) \land \dots \land
r(z_{i-1}, y)$ for fresh variables $z_1, \dots, z_{i-1}$. Rewrite
$Q^{(\ell)}$ as a $\UCQ$ and let $t$ be the maximal number of binary
atoms in one $\CQ$ in $Q^{(\ell)}$.

\begin{fact} \label{fact:bounded}
If each simple directed $r$-path in an interpretation $\Jj$ has length
at most $\ell$, then $\Jmc \models Q$ iff $\Jmc \models Q^{(\ell)}$.
\end{fact}

Fix $n = t^2$ and let $\Ii'$ be an $n$-proper colouring of $\Ii$. On
each infinite branch, select the first $\Sigma_r$-bag $\Mm$ such that
for some $\Sigma_r$-bag $\Mm'$ higher on this branch, the
$n$-neighbourhood of the element $e$, shared by $\Mm$ and its parent, is
isomorphic to the $n$-neighbourhood of the element $e'$, shared by
$\Mm'$ and its parent. The set of selected bags forms a maximal
antichain, and by K\"onig's Lemma, it is finite. Let $\Ff$ be the
interpretation obtained by taking the union of all strict ancestors of 
the selected bags, and for each element $e$ shared by a selected
$\Sigma_r$-bag and its parent, redirect each $r$-edge incident with
$e$: instead of pointing at an $r$-neighbour of $e$, it should point
at the corresponding $r$-neighbour of $e'$. If $\Kk$ is an $\ALCQt$
KB, redirect only outgoing $r$-edges, and drop the incoming
ones. Multiple nodes $e$ may be attached to the same bag $\Mm'$ in
this way, but they are only $r$-reachable from each other via elements
of $\Mm'$, so any simple $r$-path in $\Ff$ has length bounded by $2k =
\ell$.  It is routine to check that $\Ff\models \Kk$. By
Fact~\ref{fact:coloured-blocking}, $\Ff \not\models Q^{(\ell)}$; 
by Fact~\ref{fact:bounded}, $\Ff \not\models Q$. Thus, $\Ff$ is 
a finite counter-model.


%% file: transitive_atoms.tex

\section{Eliminating Transitive Atoms}~\label{sec:elim-trans}
In this section we make the hardest step towards proving our main
result: we reduce finite entailment of $\UCQt$s to finite entailment
of $\UCQ$s, for single-role ABox-trivial KBs either in $\ALCIt$ or in $\ALCQt$.
Most of the argument is not only shared for the two DLs, but
works for their common extension $\ALCIQfwdt$.  Thus, throughout this
section we let $\Kk = (\Tt,\Aa)$ be a single-role $\ALCIQfwdt$ KB with
a trivial ABox. We also fix some $Q \in \UCQt$ and a set $\Theta$ of
unary types. The overall strategy
is to show that finite counter-models can be unravelled into safe
counter-models admitting tree decompositions with bags of arbitrary
size, but with a certain more subtle width measure bounded. The bags
of these decompositions will be almost strongly connected, which will
allow dropping transitive atoms when evaluating fragments of $Q$ over
a single bag. Moreover, each bag will contain only a bounded number of
elements relevant for at-most restrictions over closures of roles,
meaning that they can be replaced with nominals (for each bag
separately). This way, the existence of bags that together form a safe
counter-model can be reduced to multiple instances of finite
entailment of $\UCQ$s in $\ALCOIt$ or $\ALCOQt$ without at-most
restrictions over closures of roles; eliminating nominals as explained
in Section~\ref{sec:ABoxesNomMain}, one arrives at the base case,
solved in the next section. As the last step, from the existence of
such safe counter-models one infers the existence of finite
counter-models.

Given that we are in the single-role case, we shall be using graph
terminology without mentioning the role name.
In particular, a \emph{source} is an element without incoming edges, a
\emph{sink} is an element without outgoing edges, an
\emph{internal} element is one that has both incoming and outgoing
edges, \new{and an \emph{isolated} element is one that has neither
incoming nor outgoing edges.}

\subsection{Unravelling Finite Counter-Models}

We relax the simplistic notion of decomposition used to eliminate
multiple roles by allowing bags to share additional elements with
their neighbours, provided that the total number of these additional
elements is bounded for each bag.

\begin{definition} [hybrid decomposition] \label{def:hybrid}
A \emph{hybrid decomposition} of an interpretation $\Ii$ is a
tree $T$ in which each node $v \in T$ is labelled with a finite
interpretation $\Ii_v = (\Delta_v, \cdot^{\Ii_v})$, called a
\emph{bag}, a set $\Gamma_v \subseteq \Delta_v$, and an element
$\rt_v\in \Delta_v$, called the \emph{root} of $\Ii_v$, such that
\begin{enumerate}
\item \label{item:hybrid-union}
  $\Ii = \bigcup_{v\in T} \Ii_v $;
\item \label{item:hybrid-connected}
  for each $e\in\Delta^\Ii$, $\left \{v \in T \bigm| e \in 
    \Delta_v\right\}$ is connected in $T$;
\item \label{item:hybrid-disjoint}
  for each edge $(u,v)$ in $T$, 
  $\Delta_u \cap \Delta_v  =  \{\rt_v\} \cup (\Gamma_u
  \cap \Gamma_v)$.
\end{enumerate}  
We let $\widetilde\Gamma_u = \left \{\rt_v \bigm | v \text{ is a child of }
  u\right\}$ and call $\max_{u\in T} |\Gamma_u|$ the \emph{width} of
$T$.
%
%
An element $d\in \Delta_u$ is \emph{fresh} in $u$ if $u$ is the root of
$T$ or $d\notin\Delta_{u'}$ for the parent $u'$ of $u$; $d$ is
\emph{local} in $u$ if $d=\rt_u$ or if 
$d\notin\Delta_v$ for all neighbours $v$ of $u$. 
\end{definition}

\noindent We often blur the distinction between the node $v$ and the
interpretation $\Ii_v$, using the term \emph{bag} for both. Note that
each element is fresh in exactly one bag, and the only local
element in $u$ that is not fresh in $u$ is $\rt_u$ (unless
$u=\varepsilon$, when $\rt_\varepsilon$ is fresh too).

While Definiton~\ref{def:hybrid} captures the fundamental structural
simplicity of our counter-models, the limitations of the target
special case impose grittier structural requirements. 

\begin{definition}[well-formedness] \label{def:well-formed}
If $\Kk$ uses both inverses and counting, a hybrid decomposition $T$ is
\emph{well-formed} for $\Kk$ if $\Ii_\varepsilon \models \Aa$ and 
for each $u\in T$, 
\begin{enumerate}

  
\item \label{item:types}
  for each concept name $A$, $A^{\Ii_u} = A^{\Ii} \cap \Delta_u$;

\item \label{item:roots}
  for each child $v$ of $u$, $\rt_v$ is a fresh sink or source in
  $\Ii_u$ \new{(or a fresh element of $\Gamma_u$)};
  
\item \label{item:external}
  for each fresh sink or source $d$ in $\Ii_u$ \new{(or fresh $ d\in
    \Gamma_u$)},  $d=\rt_v$ for exactly one child $v$ of $u$;  
  
  
\item \label{item:rel-consistent}
  for each $w\in T$, if $\rt_u \in \Gamma_w$, then $\Gamma_u \subseteq
  \Gamma_w$;

\item \label{item:rel-contained}
  for each local element $d$ in $\Ii_u$ and each concept name
  $A$ relevant for $d$: if $d \in A^{\Ii_u}$ then $d\in\Gamma_u$, and
  $\Gamma_v \cap A^{\Ii_v} \subseteq \Gamma_u$ for all $v\in T$ such
  that $\rt_v$ is  a sink in $\Ii_u$ or $\rt_u$ is a \emph{non-isolated} source in $\Ii_v$. 

\end{enumerate}
If $\Kk$ does not use inverses, ``sink or source'' in items
\ref{item:roots} and \ref{item:external} is replaced with ``sink'', and
the case ``\new{$\rt_u$ is a non-isolated source in $\Ii_v$}'' in item
\ref{item:rel-contained} is dropped, \new{but additionally we require that
each element of $\Gamma_u\setminus \{\rt_u\}$ is a sink in $\Ii_u$.}
If $\Kk$ does not use counting, we additionally require that $\Gamma_u
= \emptyset$ for all $u\in T$. 
\end{definition}

\noindent Note that in a well-formed hybrid decomposition, each
element is local in exactly one node. 

We work with infinite counter-examples because they have simpler
structure, but ultimately we need to make them finite. The following
notion guarantees that this is possible.

\begin{definition}[safety]
A hybrid decomposition $T$ is \emph{safe} if it admits no infinite
sequence of nodes $u_0, u_1, \dots$ such that for all $i$, either
$\rt_{u_{i+1}}$ is a sink in $\Ii_{u_{i}}$ or
$\rt_{u_i}$ is a \new{non-isolated} source in
$\Ii_{u_{i+1}}$.
\end{definition}

Finally, replacing $\UCQt$ with $\mathcal{UCQ}$ is possible thanks to strong
connectedness guarantees on the bags. 

\begin{definition}[well-connectedness]
An interpretation $\Jj$ is \emph{well-connected} if for all $d, e \in
\Delta^\Jj$, $(d,e) \in (r^*)^\Jj$ iff either $d=e$ or $d$ is
not a sink and $e$ is not a source in $\Jj$.  A hybrid decomposition
$T$ is well-connected if $\Ii_v$ is well-connected for all $v \in T$.
\end{definition}



In the remainder, by a \emph{$\Kk$-decomposition} we mean a hybrid
decomposition well-formed for $\Kk$, of width at most
$N^{|\RCN(\Kk)|}$ where $N$ is the counting threshold in $\Kk$. 



\begin{lemma} \label{lem:unravelling}
If $\Kk\notfentails^\Theta Q$, then there exists a counter-model admitting a
safe and well-connected $\Kk$-decomposition. 
\end{lemma}

\begin{proof}
Let $\Jj$ be a finite counter-model. We construct the special  
counter-model $\Ii$ and a witnessing $\Kk$-decomposition $T$ of $\Ii$
by unravelling $\Jj$ in a special way. We describe the construction
for the case when $\Kk$ uses inverses; if it does not, simply replace
``neighbours'' with ``direct successors''.  


Because the ABox of $\Kk$ is trivial, we have $\Ind(\Kk) = \{d\}$. We
begin from the interpretation $\Ii_\varepsilon$, obtained from $\Jj
\upharpoonright \{d\}$ by removing the only possible edge,
$\rt_\varepsilon = d$, and $\Gamma_\varepsilon = \emptyset$.
Now, roughly, for each previously added element $e$
missing a neighbour, we will be adding a new bag containing $e$ with
all its missing neighbours, along with all other elements from
$\cmp_\Jj(e)$---the strongly connected component of $e$ in $\Jj
\upharpoonright \big(\Delta^\Jj  \setminus \rel_\Jj(e)\big)$, or
$\{e\}$ if $e \in \rel_\Jj(e)$---with all their neighbours, and $\rel_\Jj(e)$.
If an element $f'$ in this bag corresponds to $f \in \Delta^\Jj$, we
call $f'$ a \emph{copy of} $f$ and $f$ the \emph{original
of} $f'$. For convenience, we extend this nomenclature to the root
bag: if $f' \in \Delta_\varepsilon$, then its original $f \in
\Delta^\Jj$ is $f'$ itself; if $f \in \Delta^\Jj$ belongs to
$\Delta_\varepsilon$, then its copy $f'$ in $\Delta_\varepsilon$ is
$f$ itself. We proceed as follows, as long as there is something to do.

For each previously added node $u$ and each \new{sink or source} $d'
\in \Delta_u$ \new{(or $d'\in\Gamma_u$)} fresh in
$u$,  we add a new child $v$ of $u$. Let
$d\in\Delta^\Jj$ be the original of $d'$ and let $X \subseteq
\Delta^\Jj$ be the set of the originals of neighbours of $d'$ in
$\Ii_u$. The interpretation $\Ii_v$ is obtained by
\begin{itemize}
\item
  taking the restriction of $\Jj$ to the subdomain comprising
 $\rel_\Jj(d)$, $\cmp_\Jj(d)$, neighbours of $\cmp_\Jj(d) \setminus
 \{d\}$, and neighbours of $d$ that do not belong to $X$;
\item
  removing all edges that are not incident with $\cmp_\Jj(d)$ and all
  edges between $d$ and elements from $X$ \new{(if $\Kk$ does not use
    inverses, drop edges outgoing from $\rel_\Jj(d) \setminus \{d\}$)};
\item
  replacing each element $e$ with a fresh copy $e'$, except that for
  elements $e \in \{d\} \cup \rel_\Jj(d)$ that already have a copy
  $e'$ in $\Ii_u$, $e'$ is reused in $\Ii_v$.
\end{itemize}
We let $\rt_v = d'$ and $\Gamma_v = \left\{e' \bigm| e \in
\rel_\Jj(d)\right\}$. 

By construction, $T$ is a hybrid decomposition of width bounded by
$\max_{d \in \Delta^\Jj} |\rel_\Jj(d)| \leq N^{|\RCN(\Kk)|}$
(Fact~\ref{fact:relevant}). It is also not difficult to check that $T$
is safe, well connected, and well formed, and that $\Ii \models \Kk$
(see Appendix C).
%
Moreover, mapping each element of $\Ii$ to its original in $\Jj$ gives
a homomorphism from $\Ii$ to $\Jj$, which implies that
$\Ii \notmodels Q$, and that $\Ii$ only realizes types from $\Theta$.
%
%
%
%
%
\end{proof}

\subsection{Evaluating Queries over Unravellings}


We aim at distributing query evaluation over bags, like when
eliminating multiple roles.  This is now harder because bags share
more than one element, but it is possible because the total number of
additional shared elements is bounded for each bag.  These elements
will be \emph{parameters} of $Q$-types.

\begin{definition}[nullary $Q$-types]
  A \emph{(nullary) $Q$-type with parameters $\Gamma$} is a set of
  pairs $(p,\eta)$ where $p$ is a fragment of $Q$ and $\eta$ is a
  partial function from $\var(p)$ to $\Gamma$.
  The \emph{$Q$-type of $\Jj$ with parameters $\Gamma$} is the set
  $\tp_{Q}^\Jj(\Gamma)$ of pairs $(p, \eta)$ where $p$ is a fragment of
  $Q$, and $\eta$ is a partial function from $\var(p)$
  to $\Gamma$ that can be extended to a matching for $p$ in $\Jj$.
\end{definition}

The power of $Q$-types with parameters is
\emph{compositionality}. Consider 
interpretations $\Jj_1,\Jj_2$ and parameter sets 
$\Gamma_1, \Gamma_2$.
If $\Gamma'\subseteq \Gamma_1$, then
\[\tp^{\Jj_1}_Q(\Gamma')
  =\tp^{\Jj_1}_Q(\Gamma_1) \upharpoonright \Gamma'\,,\]
where
$\tau_1 \upharpoonright \Gamma'  = \left\{(p, \eta \upharpoonright
  \Gamma') \bigm | (p, \eta) \in \tau_1 \right\}$ is the
\emph{projection of $\tau_1$ on $\Gamma' $}.
If $\Delta^{\Jj_1} \cap \Delta^{\Jj_2}
\subseteq \Gamma_1 \cap \Gamma_2$, then
\[\tp^{\Jj_1 \cup \Jj_2}_Q(\Gamma_1\cup\Gamma_2)
  = \tp^{\Jj_1}_Q(\Gamma_1) \oplus \tp^{\Jj_2}_Q(\Gamma_2)\,,\]
where
$\tau_1\oplus\tau_2$ is the \emph{composition} of $\tau_1$ and
$\tau_2$ defined as the set of tuples $(p, \eta)$ such that there
exist $\tau \subseteq \tau_1 \cup \tau_2$ with $\var(p') \setminus
\dom(\eta')$ pairwise disjoint for $(p',\eta')\in\tau$ and a
homomorphism $\widetilde\eta : p \to \bigcup_{(p',\eta') \in \tau}
\eta'(p')$ extending $\eta$.

Like before, we decorate elements with their $Q$-types.

\begin{definition} [$Q$-labellings]
A \emph{$Q$-labelling} for a $\Kk$-decomposition $T$ of $\Ii$ is a
function $\sigma$ that maps each element $d$ local in $u\in T$ to a
$Q$-type $\sigma(d)$ with parameters $\{d\} \cup \Gamma_u$.
We call $\sigma$ \emph{correct} if for each $d$
local in $u$, $\sigma(d)$ is the $Q$-type of $\widehat\Ii_u$ with
parameters $\{d\} \cup \Gamma_u$, where $\widehat\Ii_u$ is the
interpretation represented by the subtree of $T$ rooted at $u$. 
%
%
\end{definition}

If an interpretation $\Ii$ admits a $\Kk$-decomposition, then $\Ii
\notmodels Q$ iff the correct $Q$-labelling for any
$\Kk$-decomposition of $\Ii$ uses only $Q$-types in which no $q \in Q$
appears. Such $Q$-labellings, called \emph{$Q$-refutations}, are
what we need to find.

Given a $\Kk$-decomposition, we can (coinductively) compute the
correct $Q$-labelling bottom-up, by composing and projecting $Q$-types
of the current bag with the $Q$-types of its direct subtrees, as
captured in the following lemma.

\begin{lemma} \label{lem:correct}
\new{The correct $Q$-labelling for a $\Kk$-decomposition $T$ of an
interpretation $\Ii$ is the least (pointwise) $Q$-labelling $\sigma$ for $T$ such
that}
\[\sigma(d) = \bigg(
  \tp^{\Ii_u}_Q\big(\{d\} \cup \Gamma_u \cup \widetilde\Gamma_u\big)
  \oplus \bigoplus_{e\in \widetilde\Gamma_u} \sigma(e)
  \bigg) \upharpoonright \{d\}\cup\Gamma_u\] 
for each node $u \in T$ and each element $d$ local in $u$.
\end{lemma}

\subsection{From $\UCQt$ to $\UCQ$}

The condition in Lemma~\ref{lem:correct} provides the locality
required in automata-based decision procedures, but it still relies on
evaluating $\UCQt$s over bags. We now show how to avoid it.

\begin{definition}[localization]
  A \emph{localization} of a $\CQt$ $q$ is any $\CQ$ obtainable from
  $q$ by replacing each atom $r^* (x,y)$ either with $x=y$ or with
  $r(x,y'), r(x',y)$ for some fresh variables $x'$ and $y'$ used only in
  these two atoms.
\end{definition}

Every matching for a $\UCQt$ $q$ in an interpretation $\Jj$ extends to
a matching for some localization of $q$. So, if $\Jj$ does not satisfy
any localization of $q$, then $\Jj\notmodels q$.  Moreover, if $\Jj$
is well-connected, then each matching for a localization of $q$ in
$\Jj$ induces a matching for $q$. So, for well-connected $\Jj$,
$\Jj\notmodels q$ iff $\Jj$ does not satisfy any localization of $q$.


\begin{definition}[weak $Q$-types]
 The \emph{weak $Q$-type of $\Jj$ with parameters $\Gamma$} is the set
 $\wtp_Q^\Jj(\Gamma)$ of pairs $(p, \eta)$ such that $p$ is a fragment
 of $Q$ and $\eta$ is a partial function from $\var(p)$ to $\Gamma$
 that extends to a matching of some localization of $p$ in $\Jj$.
\end{definition}

Weak $Q$-types are over-approximations of $Q$-types,
\[\tp_{Q}^\Jj(\Gamma) \subseteq \wtp_{Q}^\Jj(\Gamma)\,,\]
and are exact for well-connected interpretations, 
\[\Jj \text { is well-connected} \implies \tp_{Q}^\Jj(\Gamma)  =
  \wtp_{Q}^\Jj(\Gamma)\,.\]

By replacing $Q$-types with weak $Q$-types in the condition of
Lemma~\ref{lem:correct}, we avoid evaluating fragments of $\UCQt$ by
evaluating their localizations instead. We also relax the condition by
replacing equality with inclusion, to facilitate the reduction to
finite (non-)entailment.

\begin{definition} [consistency] \label{def:consistency}
A $Q$-labelling $\sigma$ for a $\Kk$-decomposition $T$ of an
interpretation $\Ii$ is \emph{consistent} if 
\[\sigma(d) \supseteq \bigg(\wtp^{\Ii_u}_Q\big(\{d\} \cup \Gamma_u \cup \widetilde\Gamma_u\big) \oplus
\bigoplus_{e\in \widetilde\Gamma_u} \sigma(e)\bigg) \upharpoonright
\{d\}\cup\Gamma_u\]
for each node $u \in T$ and each element $d$ local in $u$.
\end{definition}

For any $\Kk$-decomposition $T$, the existence of a consistent
$Q$-refutation implies the existence of a correct
$Q$-refutation. Moreover, if $T$ is well-connected, then the two
conditions coincide, which eliminates false negatives in the finite
entailment decision procedure.




\begin{lemma} \label{lem:consistent}
  If a $\Kk$-decomposition $T$ of $\Ii$ admits a consistent
  $Q$-refutation, then it admits a correct $Q$-refutation.
  If $T$ is well-connected, then the correct $Q$-labelling is
  consistent.   
\end{lemma}

\begin{proof}
  Let $\sigma$ be the correct $Q$-labelling for $T$. By
  Lemma~\ref{lem:correct} and the relation between $Q$-types and weak
  $Q$-types, for each consistent $Q$-labelling $\sigma'$ for $T$,
  $\sigma(v) \subseteq \sigma'(v)$ for all $v \in T$. Consequently, if
  some consistent $Q$-refutation $\sigma'$ for $T$ exists, then $\sigma$
  is a $Q$-refutation as well. As $\sigma$ is correct, the first claim
  of the lemma follows. If $T$ is well-connected, $Q$-types and weak
  $Q$-types coincide. Combined with Lemma~\ref{lem:correct}, this
  implies that $\sigma$ is consistent. 
\end{proof}

\begin{lemma} \label{lem:there}
 If $\Kk\notfentails^\Theta Q$ then some model of $\Kk$ realizing only
 types from $\Theta$ has a safe
 $\Kk$-decomposition admitting a consistent $Q$-refutation. 
\end{lemma}

\begin{proof}
Assume $\Kk\notfentails^\Theta Q$. By Lemma~\ref{lem:unravelling},
there exists a counter-model $\Ii$ with a safe and well-connected
$\Kk$-decomposition $T$. Let $\sigma$ be the unique correct
$Q$-labelling for $T$. Because $\Ii \notmodels Q$, $\sigma$ is a
$Q$-refutation.  By Lemma~\ref{lem:consistent}, $\sigma$ is
consistent.
\end{proof}

\subsection{Recognizing Safe Counter-Examples}

We now construct an automaton $\Bb_{\Kk, Q}$ recognizing safe
$\Kk$-decompositions of models of $\Kk$, admitting consistent
$Q$-refutations. Like in the case of multiple roles, we
must give the automaton information about shared elements. For
this purpose, we enrich each node $v$ with a tuple $\up_v$ enumerating
all elements of $\Gamma_v$ that belong to the parent $u$ of $v$, and
on the edge between $u$ and $v$ we put $\rt_v$ and $\up_v$.  Again,
the actual elements in $\Delta_v$ and $\Delta_u$ do not matter:
$\rt_v$ and $\up_v$ coupled with the label on the edge from $u$ to $v$
determine which elements represented in $v$ and in $u$ should be
identified. Let $M = N^{|\RCN(\Kk)|}$ where $N$ is the counting
threshold in $\Kk$. We fix a countably infinite set $\Delta$ and
assume that the node alphabet of our automaton is the set of 4-tuples
of the form $(\Jj, \rt, \Gamma, \up)$ where $\Jj$ is a finite
interpretation with $\Delta^\Jj \subseteq \Delta$, $\rt \in
\Delta^\Jj$, $\Gamma \subseteq \Delta^\Jj$, $|\Gamma| \leq K$, and
$\up$ is a tuple enumerating a subset of $\Gamma$. The edge alphabet
is $\Delta \times \bigcup_{i=0}^M \Delta^{i}$.

The following lemma relies on $\Kk$ being expressed either in $\ALCIt$
or in $\ALCQt$.
By a base-case KB we understand a single-role KB in $\ALCIt$ or
$\ALCQt$, respectively, with a trivial ABox and no at-most
restrictions over closures of roles. Let $m=\max_{q \in Q}|q|$.

\begin{lemma} \label{lem:regularity}
  There exists an automaton $\Bb_{\Kk,Q,\new{\Theta}}$ recognizing those safe
  $\Kk$-decompositions of models of $\Kk$ realizing only types from
  $\Theta$ that admit consistent $Q$-refutations.
  The states, initial states, and accepting states of
  $\Bb_{\Kk, Q, \new{\Theta}}$ can be computed in time
  $O\big(2^{\poly( \|Q\|, \|\Kk\|, M^m)}\big)$.
  The question if $(P,q) \in
  \step_{\Bb_{\Kk,Q, \new{\Theta}}}$ can be reduced to an instance of finite
  entailment modulo types for a base 
  case KB of size at most
  $\poly(\|\Kk\|, \|Q\|, N\cdot M^m)$
  and counting threshold $N$,
  a $\UCQ$ consisting of at most
  $\|Q\|^2 \cdot M^m \cdot m^m$
  $\CQ$s of size $O(m)$,
  and a type set of size at most
  $|\Theta| \cdot 2^{\poly(\|K\|, \|Q\| , N\cdot M^m)}$.
\end{lemma}

\subsection{Making Counter-Examples Finite}

To complete the proof we need to derive finite non-entailment from the
existence of a safe counter-example.

\begin{lemma} \label{lem:back}
  If $\Kk$ is in $\ALCIt$ or $\ALCQt$ and some model of $\Kk$
  realizing only types from $\Theta$ has a safe $\Kk$-decomposition
  admitting a consistent $Q$-refutation, then $\Kk\notfentails^\Theta Q$.
\end{lemma}

For $\ALCQt$, a finite counter-model comes for free. Indeed, each safe hybrid
decomposition well-formed for a KB without inverses must be
finite, so the interpretation it represents is finite too.

For $\ALCIt$, we use the coloured blocking principle. 
By Lemma~\ref{lem:regularity}, the set of those $\Kk$-decompositions
of models of $\Kk$ that admit a weakly-consistent $Q$-refutation is
recognized by an automaton. Consequently, if nonempty, it contains a
regular tree $T$; that is, $T$ has only finitely many non-isomorphic
subtrees. From the regularity of $T$ it follows that $\Ii_v$ are
chosen from a finite set, which means that their size is
bounded. Similarly, the branching of $T$ is bounded. Finally, safety
and regularity of $T$ together imply that there is a bound on the
length of sequence of nodes $u_0, u_1, \dots, u_k$ such that for all
$i<k$, either $\rt_{u_i}$ is a source in $\Ii_{u_{i+1}}$ or
$\rt_{u_{i+1}}$ is a sink in $\Ii_{u_{i}}$. It follows that in the
interpretation $\Ii$ represented by $T$, both the branching and the
length of simple directed paths is bounded by some $\ell \in \Nat$.

%
%
%

By Fact~\ref{fact:bounded}, $\Ii\models Q$ iff $\Ii\models
Q^{(\ell)}$, where $Q^{(\ell)} \in \UCQ$ is the query defined in
Section~\ref{sec:mulroles}. Fix $t = \max_{q \in Q^{(\ell)}} |q|$ and 
$n = \max(\ell^2, t^2)$.  Let  $\Ii'$ be an $n$-proper
colouring of $\Ii$. Let $V$ be the set of nodes $v \in T$ such that
$\rt_v$ is a sink in $\Ii_{v'}$ for the parent $v'$ of $v$ and
$\rt_{v'}$ is a source in $\Ii_{v''}$ for the parent $v''$ of $v'$. By
the boundedness properties of $T$, each infinite branch of $T$ has
infinitely many nodes in $V$.  On each infinite branch, select the
first node $v \in V$ with an ancestor $\hat v \in V$ such that the
$n$-neighbourhood of $\rt_{v}$ is isomorphic with the
$n$-neighbourhood of $\rt_{\hat v}$. The set of
selected nodes forms a maximal antichain in $T$ and, by K\"onig's
Lemma, it is finite. Let $\Ff$ be the interpretation obtained by
taking the union of $\Ii_u$ for $u$ ranging over the nodes of $T$ that
have a selected descendent, except that for each selected node $v$,
$\rt_{v}$ is removed from the domain and all incoming edges are
redirected to $\rt_{\hat v}$.

\begin{lemma} \label{lem:folding-correct}
  \new{ The length of simple directed paths in $\Ff$
   is bounded by $\ell$} and $\Ff \models \Kk$.
\end{lemma}

By Fact~\ref{fact:coloured-blocking}, $\Ff \notmodels Q^{(\ell)}$,
and by Fact~\ref{fact:bounded} and Lemma~\new{\ref{lem:folding-correct}},
$\Ff \notmodels Q$.  Thus, $\Ff$ is a finite counter-model.

\subsection{Wrapping up}



  Assume $\Kk$ is in $\ALCIt$ or $\ALCQt$.
  By Lemmas~\ref{lem:there} and \ref{lem:back}, $\Kk\notfentails^\Theta Q$
  iff  some model of $\Kk$ realizing only types from $\Theta$ has a
  safe $\Kk$-decomposition admitting a  weakly consistent
  $Q$-refutation.
  Combining Lemma~\ref{lem:regularity} and Fact~\ref{fact:emptiness},
  we get a decision procedure for $\Kk\fentails^\Theta Q$ running in time 
  $O\big(\poly(|\Theta|)\cdot 2^{\poly( \|Q\|, \|\Kk\|, M^m)}\big)$ and using oracle calls to finite
  entailment modulo types for instances with parameters bounded as in Lemma~\ref{lem:regularity}.



%% file: ucqs.tex
\section{Base Case}~\label{sec:basecase}
We now solve finite entailment modulo types under the assumptions
justified in previous sections; that is, for $Q$ in $\UCQ$ and a
single-role ABox-trivial $\Kk$ either in $\ALCIt$ or in $\ALCQt$
without at-most restrictions over closures of roles. We show that it
suffices to consider tree-shaped counter-models
(Lemma~\ref{lem:tree-shaped}), and that they consitute an effectively
regular set (Lemma~\ref{lem:reg}).  The decision procedure then
amounts to computing the automaton and testing its emptiness; it runs
in time $\poly(|\Theta|, 2^{\|\K\|^2\cdot\|Q\|})$ where $\Theta$ is the
set of allowed types.

An interpretation $\Jj$ is \emph{tree-shaped} if its elements can be
arranged into a tree such that the unique $f\in\Ind(\Kk)$ is the root
and $r$-edges in $\Jj$ are allowed only between parents and children;
if $\Kk$ is an $\ALCQt$ KB, $r$-edges must point down.

\begin{lemma} \label{lem:tree-shaped}
  $\Kk\notfentails^\Theta Q$ iff there exists a tree-shaped
  counter-model of bounded degree realizing only types from $\Theta$. 
\end{lemma}

\begin{proof}

Let $\Ii$ be a finite counter-model realizing only types from
$\Theta$.  We get a tree-shaped counter-model $\Ii'$ by unravelling $\Ii$
from the unique $f \in \Ind(\Kk)$. If $\Kk$ is in $\ALCQt$, we unravel by
adding fresh copies of all direct successors of previously copied
elements from $\Ii$. If $\Kk$ is in $\ALCIt$, we add fresh copies of
all direct successors and fresh copies of all direct predecessors. It is
clear that $\Ii'$ is a tree-shaped counter-model realizing only types
from $\Theta$, of degree bounded by twice the degree of $\Ii$.

Conversely, let $\Ii$ be a tree-shaped counter-model of bounded degree
realizing only types from $\Theta$.

If $\Kk$ is in $\ALCQt$, we enrich $\Ii$ to ensure that witnesses for
at-least restrictions over $r^*$ are never glued with each other. Note
that in this case all $r$-edges in $\Ii$ point down. Let $N$ be the
counting threshold in $\Kk$. For each concept name $A$, introduce
fresh concepts names $A_1, A_2, \dots, A_N$, called the {\em shades}
of $A$.  Extend $\Ii$ to these concept names in such a way that
$A_1^\Ii, A_2^\Ii, \dots, A_N^\Ii$ form a partition of $A^\Ii$, and
for each $n \leq N$, if $d\in\Delta^\Ii$ has at least $n$ successors
in $A$, then it has successors in at least $n$ shades of $A$.
This can be done greedily, by processing $\Ii$ top down, ensuring that
subtrees rooted at unprocessed nodes use each shade of $A$ at most
once. Take a shallowest unprocessed node $v$. For $n=1, 2, \dots, N$,
if $v$ has at least $n$ descendents in $A$ but only $n-1$ are
painted with a shade of $A$, pick an unpainted one and paint it with
an unused shade of $A$. After the whole tree is processed, paint each
remaining element in $A$ with an arbitrary shade of $A$.

For each concept name $B$ (including the shades) enrich $\Ii$ further
by introducing a fresh concept name $B'$ with extension $(\exists
r^*.B)^\Ii$ and add $B' \sqsubseteq \exists r^*.B$ to $\Kk$; if
$\Kk$ is in $\ALCIt$, introduce also $B''$ with extension $(\exists
(r^-)^*.B)^\Ii$ and add $B'' \sqsubseteq \exists (r^-)^*.B$.
Let $n=\max(2, |Q|^2)$ and let $\Ii'$ be an $n$-proper colouring of
$\Ii$. Consider a level $l$ in $\Ii'$ such that all $n$-neighbourhoods
realized in $\Ii'$ are already realized above $l$. For each element
above level $l$, choose a witness for each CI of the form $A
\sqsubseteq \exists r^* .B$ or $A \sqsubseteq \exists (r^-)^*. B$ in
$\Kk$. Restrict the domain to the nodes above level $l$ and the chosen
witnesses, together with the paths that lead to them. Redirect each
edge leaving the restricted domain to some element above level $l$,
preserving the $n$-neighbourhood.

The only nontrivial thing to check is that at-least restrictions are
not violated, in the case when $\Kk$ is in $\ALCQt$. For restrictions
over $r$, this is because siblings have different colours in
each $n$-proper colouring, so edges leading to different siblings are
never redirected to the same element. For restrictions over $r^*$,
this is because witnesses for the CI $A_i' \sqsubseteq \exists r^*. A_i$ are
preserved for each shade $A_i$ of $A$.
\end{proof}

\begin{lemma} \label{lem:reg}
  The set of tree-shaped counter-models realizing only types from
  $\Theta$ is recognized by an automaton computable in
  time $2^{O(\|\K\|^2\cdot\|Q\|)}$.
\end{lemma}

This is a routine construction. For $\ALCQt$, the node alphabet is
$\Theta$ and the edge alphabet is trivial. To verify that the input
tree is a model of $\Kk$, the automaton stores the unary type of the
parent of the current node, and for each $A\in \CN(\K)$, it stores in
the state the minimum number of elements in $A$ to be found in the
current subtree, together with a binary flag indicating if progress
has been made recently in finding them. This component of the
automaton has $(4\cdot N)^{|\CN(\K)|}$ states, where $N$ is the
counting threshold in $\Kk$. Verifying that the input tree does not
satisfy $Q$ involves storing a set of subqueries of CQs from $Q$,
which are not to be satisfied in the current subtree. This component
has $2^{\|Q\|}$ states. The whole automaton is the product of the two
components. It is easy to see that the step relation can be computed
in time polynomial in the number of states.

The case of $\ALCIt$ is similar, except that the edge
alphabet is $\{r, r^-\}$ and the first component has separate
information about elements reachable and backwards reachable from
current node in the current subtree, and also outside (without
progress flags). The first component thus has $2^{7\cdot|\CN(\Kk)|}$
states, because it only needs to count up to 1.

Note that if the automaton accepts any tree, it also accepts a regular
one, and a regular tree has bounded degree. 


%% file: conclusions.tex
\section{Outlook}

This paper provides first positive results on finite entailment of
non-local queries over DLs knowledge bases. The main technical
contribution is optimal \twoexp upper bounds for finite entailment of
$\UCQt\!$s over \ALCOIt and \ALCOQt knowledge bases. To obtain these
results, we have shown intermediate reductions that are interesting in
their own, and could be applied to similar settings.

\smallskip
 There are several directions to follow for future work. A first
possibility is to vary the DL language.  One could consider
lightweight DLs from the $\mathcal{EL}$ and \textsl{DL-Lite} families
or extensions of \ALCOIt with e.g. role inclusions. For
the latter, a different approach to the one proposed here is needed
because our techniques rely on the lack of interaction between
different roles. Another option is to allow for controlled interaction
of inverses and number restrictions as e.g.\ in \ALCIQfwdt. A second
possibility is to consider more expressive non-local queries, such as
positive regular path queries. In this case new techniques seem to be
needed, e.g.\ the coloured blocking principle does not
work for PRPQs.


%% file: appendix-nominals.tex
\newcommand{\rootint}{\ensuremath{\Ii_0}}
\newcommand{\full}{\mathsf{full}}
\newcommand{\globsuc}{\mathsf{globsuc}}

\section{Eliminating Nominals, Trivializing ABoxes}~\label{sec:ABoxesNom}
We prove three results: for the logic $\ALCOIt$, for a restricted variant of $\ALCOQt$ logic (without at-most restrictions over transitive closures of roles), and for the full $\ALCOQt$ logic.
We start with the simplest proof, and then describe additional components needed in the subsequent proofs.

\subsection{$\ALCOIt$}
As there is no counting in the logic, for a role $r$ and a concept name $A$, we will use the standard notation
\[ \forall r. A\,, \quad \forall r^*. A\,, \quad \exists r. A\,, \quad \exists r^*. A\,, \]
instead of
\[ \qnrleq 0 r {\overline{A}}\,, \quad \qnrleq 0 {r^*} {\overline{A}}\,, \quad \qnrgeq 1 r A\,, \quad \qnrgeq 1 {r^*} A\,, \]
respectively.
Note that restrictions of the form $\qnrleq n {r^*} A$ for $n = 0$ can be eliminated using stickiness, and for $r$ being an inverse role they are not present in the normal forms that we assume throughout the paper.
In this section we will use them for the sake of brevity, highlighting the places where it might create confusion.

Let $\Kk = (\Tt, \Aa)$ be an $\ALCOIt$ knowledge base, $Q \in \UCQt$ with $\CQt$s of size at most $m$.
One can decide if $\Kk \fentails Q$ in time
$\poly(\|\Kk\|, (2m)^{|Q|\cdot\|\Kk\|^{O(m)}})$
using oracle calls to finite entailment over $\ALCIt$ KBs of size at most
$\poly(\|\Kk\|)$,
with trivial ABoxes and unions of at most
$|Q| \cdot \|\Kk\|^{O(m)}$
$\CQt$s of size $O(m)$.

We present the result as an oracle reduction; however, it is worth noting that it can be seen as a \emph{truth-table reduction}~\cite{post1944,LADNER1975103}; that is, the calls to the oracle are not adaptive.

\subsubsection{Modifying the KB}
\begin{itemize}
 \item Replace each transitive role atom $r^*(a, b)$ in the ABox with a CI in the TBox: $\{a\} \sqsubseteq \exists r^*. \{b\}$.
 \item For each nominal $\{a\}$ and role $r$ introduce fresh concept names $A_a \equiv \{a\}$, $A_{r,a} \equiv \exists r. \{a\}$ and $A_{r^*, a} \equiv \exists r^*. \{a\}$.
 \item Normalize the knowledge base. We will denote all concept names present in the KB after the normalization as \emph{original} concept names.
 \item Introduce \emph{auxiliary} concept names: for each original concept name $A$ and each role $r$ introduce concept names $C_{\exists r. A}$ and $C_{\exists r^*. A}$, with the following axiomatization:
       \[ C_{\exists r. A} \sqsubseteq \exists r. A\,,\quad \overline{C_{\exists r. A}} \sqsubseteq \forall r. \overline{A}\,, \]
       \[ C_{\exists r^*. A} \sqsubseteq \exists r^*. A\,,\quad \overline{C_{\exists r^*. A}} \sqsubseteq \forall r^*. \overline{A}\,. \]
       To preserve the normal form, the last CI is split into two:
       $\overline{C_{\exists r^*. A}} \sqsubseteq \overline{A}$ and
       $\overline{C_{\exists r^*. A}} \sqsubseteq \forall r. \overline{C_{\exists r^*. A}}$.
\end{itemize}
From now on, let $\Kk$ denote the KB after above modifications.

\subsubsection{Unravelling}
\begin{definition}
Let $\Ii$ be a model of $\Kk$ and $x \in \Delta^\Ii$.
A \emph{witnessing set} for $x$ is a set $S$ of pairs $(y, r)$, where $y \in \Delta^\Ii$ and $r$ is a role, such that:
\begin{itemize}
 \item for every $(y, r) \in S$, $y$ is a direct $r$-successor of $x$;
 \item for each CI of the form $A \sqsubseteq \exists r. B$ such that $A \in \tp^{\Ii}(x)$, there exists $(y, r) \in S$ such that $y$ is in concept $B$;
 \item for each CI of the form $A \sqsubseteq \exists r^*. B$ such that $A \in \tp^{\Ii}(x)$, either $B \in \tp^{\Ii}(x)$ or there exists $(y, r)\in S$ such that $y$ is in concept $C_{\exists r^*. B}$.
\end{itemize}
\end{definition}
Observe that in any minimal witnessing set there are at most as many elements as there are CIs in the TBox.
We will call an element $y$ a \emph{required witness} for $x$ if $y$ is present in every witnessing set for $x$.

\begin{lemma}
If $\Kk \not\fentails Q$, there exists a counter model $\Imc$ which can be decomposed into domain-disjoint intepretations $\Ii\upharpoonright \Ind(\Kk)$ and $\Ii_1, \dots, \Ii_n$ for some $n \leq \|\Kk\| \cdot |\Ind(\Kk)|$ with two kinds of additional edges:
arbitrary edges connecting elements from $\bigcup_i\Delta^{\Ii_i}$ with elements from $\Nom(\Kk)$, and a single edge connecting a distinguished element $d_i\in\Delta^{\Ii_i}$ with a corresponding element $e_i \in \Ind(\Kk)$ for each $i\leq n$.
Moreover, one can assume that:
\begin{itemize}
 \item if $e_i$ is not a nominal, $e_i$ is not a required witness for $d_i$,
 \item for any nominal $a$, 
 $\{(b, r): b \in \Ind(\Kk), (a, b) \in r^\Ii\} \cup \{(d_i, r_i): e_i = a\}$ is a witnessing set for $a$.
\end{itemize}
\end{lemma}
\begin{proof}
This can be achieved by routine one-step unravelling from the individuals.
Let $\Jj$ be any counter-model.
$\Ii$ is constructed as follows.
Take a copy of $\Jj \upharpoonright \Ind(\Kk)$.
For every individual $x \in \Ind(\Kk)$, take any minimal witnessing set $S$ for $x$ in $\Jj$.
For each $(y, r) \in S$, add a fresh copy of the whole interpretation $\Jj$ to $\Ii$, identifying the nominals, and add an $r$-edge from $x$ to the newly created copy of $y$.
As each element and each edge in $\Ii$ was created as a copy of some element or edge in $\Jj$, there is a homomorphism from $\Ii$ to $\Jj$ witnessing that $\Ii \not\models Q$.
\end{proof}

\subsubsection{Constructing KBs for $\Ii_i$.}
We search for a counter-model admitting the described decomposition.
We will refer to ${\Ii\upharpoonright \Ind(\Kk)}$ as $\rootint$.
Iterate through all possible interpretations $\rootint$ \comment{($4^{|\Ind(\Kk)|^2\cdot|\Rol(\Kk)|}\cdot 2^{|\Ind(\Kk)|\cdot|\CN(\Kk)|}$ options)} and perform the rest of the procedure for a fixed $\rootint$; a counter-model exists iff it exists for some of these choices.

We will construct KBs representing abstractly $\Ii_i$ for $1 \leq i \leq n$; these KBs will have some additional concept names: $D_{\exists r. A} \equiv \exists r. A$ and $D_{\exists r^*. A} \equiv \exists r^*. A$, for each role $r$ and original concept name $A$.
Their meaning is similar to the auxiliary concept names, but they are supposed to represent having appropriate (direct) $r$-successor \emph{inside one interpretation $\Ii_i$}, while the auxiliary concept names speak about the whole interpretation $\Ii$.
We will call them \emph{local auxiliary concept names}.
For any choice of a unary type $\tau$ containing both concept names from $\Kk$ and the local auxiliary concept names (or their complements), let $\Kk_\tau$ be the KB $\Kk$ with the following modifications:
\begin{itemize}
 \item the ABox is $\left\{A(b) \bigm | A \in \tau\right\}$ for a fresh individual $b$;
 \item remove the axiomatization of concept names $A_{r, a}$;
 \item replace the axiomatization of the auxiliary concept names:
       \[ C_{\exists r. A} \equiv \exists r. A \sqcup \bigsqcup_{\substack{a \in \Nomi(\Kk)\\A(a) \in \Aa}} A_{r, a}\,, \]
       \[ C_{\exists r^*. A} \equiv \exists r^*. A \sqcup \bigsqcup_{\substack{a \in \Nomi(\Kk)\\ C_{\exists r^*. A}(a) \in \Aa}} \exists r^*. A_{r, a}\,, \]
 \item for each nominal $\{a\}$, each $A$ such that $A(a) \in \Aa$ and each role $r$:
 \begin{itemize}
  \item for each CI of the form $A \sqsubseteq \forall r. B$ add a new CI $A_{r^-, a} \sqsubseteq B$ to the TBox;
  \item for each CI of the form $A \sqsubseteq \forall r^*. B$ add a new CI $A_{r^-, a} \sqsubseteq \forall r^*. B$ to the TBox;
 \end{itemize}
 \item remove all CIs with a nominal on the left-hand side;
 \item replace CIs by the following rules, for any original concept names $A$, $B$ and any role $r$ (by the previous point, $A$ is not a nominal):
 \begin{itemize}
  \item $A \sqsubseteq \exists r. B \longrightarrow A \sqsubseteq C_{\exists r. B}$;
  \item $A \sqsubseteq \forall r. B \longrightarrow A \sqsubseteq \overline{C_{\exists r. \overline{B}}}$;
  \item $A \sqsubseteq \exists r^*. B \longrightarrow A \sqsubseteq C_{\exists r^*. B}$;
  \item $A \sqsubseteq \forall r^*. B \longrightarrow A \sqsubseteq \overline{C_{\exists r^*. \overline{B}}}$;
 \end{itemize}
 \item for any CI of the form $\bigsqcap_i A_i \sqsubseteq \bigsqcup_j B_j$, remove all nominals from the right-hand side (i.e. if $B_j$ is a nominal, remove $B_j$ from the disjunction);
 \item add the axiomatization of concept names $D_{\exists r. A}$ and $D_{\exists r^*. A}$.
\end{itemize}

Iterate through all possible choices of a number ${n \leq \|\Kk\| \cdot |\Ind(\Kk)|}$, and sequences of: unary types $\tau_i$ (containing concept names from $\Kk$ and the local auxiliary concept names), individuals $e_i$ and roles $r_i$ for $i = 1, \dots, n$.
\comment{($\|\Kk\|\cdot |\Ind(\Kk)| \cdot 4^{|\CN(\Kk)|\cdot \|\Kk\| \cdot |\Ind(\Kk)|} \cdot (|\Ind(\Kk)| \cdot 2|\Rol(\Kk)|)^{\|\Kk\| \cdot |\Ind(\Kk)|}$ options)}
We will consider an interpretation obtained by taking a union of $\rootint$ and all $\Ii_i$ for $i = 1, \dots, n$, where $\Ii_i$ is any model of $\Kk_{\tau_i}$, and adding two kinds of additional edges: between $e_i$ and $d_i$ (the only individual of $\Kk_{\tau_i}$) by role $r_i$ and from all elements in concept $A_{r, a}$ to $a$ by role $r$.
Perform the rest of the procedure for a fixed choice of $n$ and all $\tau_i$, $e_i$ and $r_i$ for $i = 1, \dots, n$; there is a counter-model iff there is a counter-model consistent with one of such choices.

We will see that the following procedure verifying that the resulting interpretation is a model of $\Kk$ does not depend on the choices of models $\Ii_i$.
Regardless of the exact shape of $\Ii_i$, the local auxiliary concept names $D_{\exists r^*. A_{r^-, a}}$ present in $\tau_i$ determine the set of nominals $r$-reachable from $d_i$ through $\Ii_i$.
Knowing also $\Ii_0$, we can calculate for each individual the set of $r$-reachable nominals and verify that it corresponds to the auxiliary concept names $C_{\exists r^*. A_{a}}$ present in the unary type of this individual.
This reachability information is the crucial part that needs to be verified; the rest is a number of straightforward conditions. Namely, one needs to check that:
\begin{itemize}
 \item for each individual, each auxiliary concept name $C_{\exists r^*. A}$ present in its unary type is witnessed, either by an $r$-reachable element of $\Ii_0$ in concept $A$ or by the local auxiliary concept name $D_{\exists r^*. A}$ present in some $\tau_i$ for an $r$-reachable $d_i$;
 \item for each individual, verify that it belongs to each auxiliary concept of the form $C_{\exists r. A}$ and $\overline{C_{\exists r. A}}$ present in its unary type; each of its direct successors is either inside $\Ii_0$ or is some $d_i$ (with unary types $\tau_i$) for $1 \leq i \leq n$, so this can be done by simply checking direct successors one by one;
 \item CIs of the form $\bigsqcap_i A_i \sqsubseteq \bigsqcup_j B_j$ are satisfied in all elements of $\Ii_0$;
 \item no universal CIs are violated in $\Ii_0$ and by the edges between $e_i$ and $d_i$; in particular, for any element being subject to restriction $\forall {r^*}. A$ (i.e. in concept $\overline{C_{\exists {r^*}. \overline{A}}}$), all its direct $r$-successors must be subject to the same restriction (stickiness).
\end{itemize}

Every counter-model admitting the described decomposition can be constructed this way.
If the verification procedure fails, there is no counter-model consistent with the choices made.
If it succeeds, the only thing left to do is to check if there are models of $\Kk_{\tau_i}$ for which $Q$ does not have a match in the resulting interpretation $\Ii$.

\subsubsection{Adjusting the query.}
We want to reduce reasoning about matches of $Q$ in the whole interpretation $\Ii$ to reasoning about matches of some query in the interpretations $\Ii_i$.
To achieve this, for each $q \in Q$ we will construct a $\UCQt$ $Q'$ mentioning $d_i$ for $i = 1, \dots, n$ as individuals such that, for any interpretation $\Ii$ admitting the described decomposition, $\Ii \models q$ iff there is some $q' \in Q'$ such that each connected component of $q'$ has a match in some $\Ii_i$ for some $i = 1, \dots, n$.
Each disjunct $q' \in Q'$ will correspond to some way in which a match of $q$ can be distributed among interpretations $\Ii_i$ for $i = 0, 1, \dots, n$.

Existence of an $r$-path between any two individuals $a, b$ can be deduced from $\rootint$ alone: any $r$-path from $a$ to $b$ is either contained entirely in $\rootint$ (which can be easily verified), or it passes through at least one nominal, so its existence can be checked by comparing the sets of nominals $r$-reachable from $a$ and $r^-$-reachable from $b$ (these sets in turn can be deduced from the unary types of $a$ and $b$).

\begin{definition}
Let $\Imc$ be an interpretation and $q \in \CQt$.
An \emph{enhanced match} for $q$ in $\Imc$ is a match $\eta:\var(q)\to\Delta^\Imc$ along with \emph{witnessing paths}: for each transitive atom $r^*(x, y)$ in $q$ there is one chosen $r$-path in $\Imc$ from $\eta(x)$ to $\eta(y)$.
\end{definition}

The crucial observation is that any path witnessing a transitive atom $r^*(x, y)$ is either contained entirely in one of the interpretations $\Ii_i$ for $0 \leq i \leq n$, or can be split into three parts: $r^*(x, a) \land r^*(a, b) \land r^*(b, y)$, where $a$ and $b$ are individuals and the paths witnessing the first and the third atom do not pass through individuals (i.e. each of them is contained inside one interpretation $\Ii_i$ for $1 \leq i \leq n$).
Thus, a $\UCQt$ $Q'$ capturing all possible distributions of an enhanced match of some $q \in Q$ is a disjunction of all possible results of the following procedure:
\begin{itemize}
 \item for each transitive atom $r^*(x, y)$ in $q$, either leave it unmodified or choose two individuals $a, b \in \Ind(\Kk)$ and replace it with $r^*(x, a) \land r^*(a, b) \land r^*(b, y)$;
 \item choose a subset of variables, a mapping from this subset to $\Ind(\Kk)$ and replace the chosen variables with individuals in the whole query (we assume that no other variable will be mapped to an individual);
 \item evaluate all atoms involving only individuals (this can be done, as $\Ii_0$ is already fixed, and the reachability between the individuals is known);
 \item replace atoms involving one individual from $\Ind(\Kk)$ with equivalent ones mentioning either some $d_i$ or no individual at all; this requires considering several simple cases, which we will omit here, such as: for each atom $r(x, a)$ where $a$ is an individual, choose to replace it either with $A_{r, a}(x)$ (only if $a$ is a nominal) or with $x = d_i$ for some $i$ such that $e_i = a$.
\end{itemize}

Thus, if $m$ is the number of transitive atoms in $q$, we can replace $q$ with a disjunction of $\CQt$s $q_1, \dots, q_k$ for $k \leq (|\Ind(\Kk)|^2+1)^m \cdot (|\Ind(\Kk)|+n+1)^{|\var(q)|}$, and assume that each connected component of each $q_j$ is contained entirely inside one interpretation $\Ii_i$ for $1 \leq i \leq n$.

Let $Q'$ be the $\UCQt$ being a disjunction of all resulting $\CQt$s for all $q \in Q$.


Now we just need a standard reduction of finite entailment to the case where one assumes that each $\CQt$ is connected.
In any counter-model, for each of $q \in Q'$, for at least one connected component of $q$ there must be no match (contained entirely in $\Ii_i$ for some $1 \leq i \leq n$).
For each possible choices of one connected component from each $q \in Q'$, check the entailment of their disjunction in each $\Kk_{\tau_i}$ for $1 \leq i \leq n$.
There is a counter-model for $Q$ and $\Kk$ iff for some choice of the connected components there are counter-models for the resulting $\UCQt$ and each $\Kk_{\tau_i}$ for $1 \leq i \leq n$; there are at most $(2m)^{k\cdot |Q|}$ such choices, where $k$ is the bound on the number of $\CQt$s $q_1, \dots, q_k$ above.

\subsection{$\ALCOQt$ without at-most restrictions over transitive closures}
We present only the modifications needed to be done in the proof for $\ALCOIt$.

Let $\Kk = (\Tt, \Aa)$ be an $\ALCOQt$ knowledge base without at-most restrictions over transitive closures of roles, with counting threshold $N$, and $Q \in \UCQt$ with $\CQt$s of size at most $m$.

%
One can decide if $\Kk \fentails Q$ in time
$\poly(\|\Kk\|, (2m)^{|Q|\cdot (N \cdot \|\Kk\|)^{O(m)}})$
using oracle calls to finite entailment modulo types over $\ALCQt$ KBs of size at most
$\poly(N, \|\Kk\|)$,
with trivial ABoxes and unions of at most
$|Q| \cdot (N \cdot \|\Kk\|)^{O(m)}$ 
$\CQt$s of size
$O(m)$.

\subsubsection{Auxiliary concepts.}
Introduce new auxiliary concept names $C_{\qnrleq k r A}$, $C_{\qnrgeq k r A}$ and $C_{\qnrgeq k {r^*} A}$ for each $k < N$, role name $r$ and original concept name $A$; they are supposed to denote that, if an element $x$ is in such a concept, it belongs to the concept in the subscript (in the whole interpretation $\Ii$).
That is, their axiomatization is simply $C_{\qnrleq k r A} \sqsubseteq {\qnrleq k r A}$, $C_{\qnrgeq k r A} \sqsubseteq {\qnrgeq k r A}$, $C_{\qnrgeq k {r^*} A} \sqsubseteq {\qnrgeq k {r^*} A}$.

\subsubsection{Unravelling.}
The unravelling procedure lifts easily to the case with counting, without inverses.
Assume that the auxiliary concept names have the intended meaning in the original counter-model; for example, an element is in concept $C_{\qnrgeq k {r^*} A}$ iff it has at least $k$ $r$-successors in concept $A$.
The definition of a witnessing set for $x$ is modified to be $S$ such that:
\begin{itemize}
 \item for every $(y, r) \in S$, $y$ is a direct $r$-successor of $x$;
 \item for each CI of the form $A \sqsubseteq \qnrgeq k r B$ such that $A \in \tp^{\Ii}(x)$, there exist at least $k$ distinct $(y, r) \in S$ such that $y$ is in concept $B$;
 \item for each CI of the form $A \sqsubseteq \qnrgeq k {r^*} B$ such that $A \in \tp^{\Ii}(x)$:
 \begin{itemize}
  \item if $B \in \tp^{\Ii}(x)$, at least $k-1$ distinct $r$-successors of $x$ in concept $B$ are $r$-reachable from the elements $\{y: (y, r) \in S\}$,
  \item if $B \not\in \tp^{\Ii}(x)$, at least $k$ distinct $r$-successors of $x$ in concept $B$ are $r$-reachable from the elements ${\{y: (y, r) \in S\}}$.
 \end{itemize}
\end{itemize}
\comment{(any minimal witnessing set has size at most $N \cdot \|\Kk\|$)}

Let $\Jj$ be any counter-model.
A counter-model $\Ii$ admitting the desired decomposition is constructed as follows.
Take a copy of $\Jj \upharpoonright \Ind(\Kk)$.
For every individual $x \in \Ind(\Kk)$, take any minimal witnessing set $S$ for $x$ in $\Jj$.
For each $(y, r) \in S$, add a fresh copy of the whole interpretation $\Jj$ to $\Ii$, identifying the nominals and \emph{removing all outgoing edges from the nominals} and add an $r$-edge from $x$ to the newly created copy of $y$.
As each element and each edge in $\Ii$ was created as a copy of some element or edge in $\Jj$, there is a homomorphism from $\Ii$ to $\Jj$ witnessing that $\Ii \not\models Q$.

To argue that $\Ii \models \Kk$, we will show that the satisfaction of all the required restrictions is preserved.
Recall that there are no inverses and no at-most restrictions over transitive closures.
For any element $x$ in $\Ii$ created as a copy of some $x'$ in $\Jj$, and for any role name $r$, all direct $r$-successors of $x$ were created as copies of distinct direct $r$-successors of $x'$, which preserves all at-most restrictions.
All required restrictions of the form $\qnrgeq k r A$ are clearly satisfied in $\Ii$, as the copies of appropriate elements and edges are added explicitly during the procedure.
As for restrictions of the form $\qnrgeq k {r^*} A$, consider an individual that has $m$ $r$-reachable elements in concept $A$ in $\Jj$:
\begin{itemize}
 \item if $m < N$, at least one copy of each of them is $r$-reachable in $\Ii$ from this individual (recall the assumption that the concepts $C_{\qnrgeq k {r^*} A}$ have the intended meaning in $\Jj$ and the definition of the witnessing set);
 \item if $m \geq N$, then copies of at least $N-1$ of them are reachable (by the argument above), which is enough to satisfy any counting restriction.
\end{itemize}
Therefore, for any element $x$ in $\Ii$ created as a copy of $x'$ from $\Jj$, any role name $r$ and any concept name $A$, it can be easily seen that either:
\begin{itemize}
 \item there is a nominal in concept $C_{\qnrgeq {N-1} {r^*} A}$ $r$-reachable from $x$;
 \item at least one copy of every element in concept $A$ $r$-reachable from $x'$ (in $\Jj$) is $r$-reachable from $x$ (in $\Ii$).
\end{itemize}
This is enough for all required restrictions to be satisfied in $\Ii$.

Note that, as there are no inverse roles, the edges between $e_i$ and $d_i$ are always directed towards $d_i$.


\subsubsection{Local auxiliary concepts.}
For $\ALCOIt$, when describing KBs $\Kk_{\tau_i}$, we defined local auxiliary concepts $D_{\exists r. A}$ and $D_{\exists r^*. A}$.
Here we will define analogous concepts $D_{\qnrleq k r A}$, $D_{\qnrgeq k r A}$ and $D_{\qnrgeq k {r^*} A}$, although this time we additionally need to be prepared to distinguish elements $r_i$-reachable from $d_i$ to avoid counting them twice later in the process.

For each role name $r$ we introduce a concept name $R_r$; if an element $x$ in $\Ii_i$ is $r$-reachable from $d_i$, it is supposed to be in this concept.
The axiomatization is simple: $R_r(d_i)$ is present in the ABox, and $R_r \sqsubseteq \forall r. R_r$ is present in the TBox.
Then, for each original concept name $A$, introduce a concept name $C_{A\sqcap \overline{R_r}} \equiv A \sqcap \overline{R_r}$.

Introduce local auxiliary concept names $D_{\qnrleq k r A}$, $D_{\qnrgeq k r A}$ and $D_{\qnrgeq k {r^*} A}$ for each $k < N$, role name $r$ and $A$ being either an original concept name or $C_{B\sqcap \overline{R_r}}$ for an original concept name $B$; if an element $x \in \Delta^{\Ii_i}$ is in such concept, it belongs to the concept in the subscript in $\Ii_i$.

When constructing KBs $\Kk_\tau$, rewrite CIs to use the auxiliary concept names; that is, replace:
\begin{itemize}
 \item $A \sqsubseteq \qnrleq k r A$ with $A \sqsubseteq C_{\qnrleq k r A}$;
 \item $A \sqsubseteq \qnrgeq k r A$ with $A \sqsubseteq C_{\qnrgeq k r A}$;
 \item $A \sqsubseteq \qnrgeq k {r^*} A$ with $A \sqsubseteq C_{\qnrgeq k {r^*} A}$.
\end{itemize}
Also, remove the axiomatization of the auxiliary concepts altogether. 

\subsubsection{Verifying being a model of $\Kk$.}
For each choice of $\Ii_0$, $n$, $e_i$, $r_i$, $\tau_i$ for $i = 1, \dots, n$, we need to verify whether there exists a counter-model admitting a decomposition consistent with their values.
For $\ALCOIt$, to perform this verification, we solved multiple instances of finite entailment problem for $\Kk_{\tau_i}$ for each $i = 1, \dots, n$ and the modified query $Q'$.
Now we will need to solve instances of finite entailment modulo types; we will define sets $\Theta_i$ (which will depend on $\Ii_0$, $n$, $e_i$, $r_i$ and $\tau_i$) of unary types allowed for $\Kk_{\tau_i}$ to make sure that whenever an element of a model $\Ii_i$ of $\Kk_{\tau_i}$ is in concept $C_{\qnrgeq k {r^*} A}$, it is also in concept $\qnrgeq k {R^*} A$ in the whole interpretation $\Ii$.

For a given element $x$ of some $\Ii_i$, we will be talking about \emph{local successors}: successors reachable from $x$ without leaving the model $\Ii_i$ (without passing through any nominal) and \emph{global successors}: ones that are reachable by paths passing through some nominal.
Note that the sets of local successors and global successors are not necessarily disjoint.

For a nominal $a$, define $\globsuc_{a,r,A}$ to be the set of $r$-successors of $a$ in concept $A$.
Note that, even though this exact set depends on choices of models $\Ii_i$, the local auxiliary concepts $D_{\qnrgeq k {r^*} A}$ present in $\tau_i$ for $i = 1, \dots, n$ along with $\Ii_0$ give a lower bound on its cardinality.
Moreover, assuming that they have the intended meaning in the interpretations $\Ii_i$, one can deduce from them (and from $\Ii_0$) the exact cardinality of $\globsuc_{a,r,A}$ up to $N$.
For a set of nominals $S$, let $\globsuc_{S,r,A} = \bigcup_{a \in S} \globsuc_{a,r,A}$.

For each $i = 1, \dots, n$ define the set $\Theta_i$ of allowed unary types $\tau$ such that, for every $k < N$, role name $r$ and original concept name $A$:
\begin{itemize}
 \item if $C_{\qnrgeq k r A} \in \tau$, then $\max\{\ell: D_{\qnrgeq \ell r A} \in \tau\} + |\{a \in \Nom(\Kk): A_{r,a} \in \tau \land A(a) \in \Aa\}| \geq k$;
 \item if $C_{\qnrleq k r A} \in \tau$, then $\min\{\ell: D_{\qnrleq \ell r A} \in \tau\} + |\{a \in \Nom(\Kk): A_{r,a} \in \tau \land A(a) \in \Aa\}| \leq k$;
 \item if $C_{\qnrgeq k {r^*} A} \in \tau$, then:
 \begin{itemize}
  \item if $d_i$ is not $r$-reachable from an element of type $\tau$ through any nominal (this depends only on $\Ii_0$, $r_i$ and the set of nominals $r$-reachable from an element of type $\tau$), then the sets of local and global $r$-successors in concept $A$ are disjoint, so the condition is:
  $\max\{\ell: D_{\qnrgeq \ell {r^*} A} \in \tau\} + |\globsuc_S| \geq k$, where $S = \{a \in \Nom(\Kk): C_{\qnrgeq 1 {r^*} {A_a}} \in \tau\}$;
  \item if $d_i$ is $r$-reachable from an element of type $\tau$ through any nominal, then these sets are intersecting, but we can avoid double counting thanks to the concept names $C_{A \sqcap \overline{R_r}}$ introduced in $\Kk_{\tau_i}$:
  $\max\{\ell: D_{\qnrgeq \ell {r^*} {C_{A \sqcap \overline{R_r}}}} \in \tau\} + |\globsuc_S| \geq k$, where $S = \{a \in \Nom(\Kk): C_{\qnrgeq 1 {r^*} {A_{r,a}}} \in \tau\}$.
 \end{itemize}
\end{itemize}
We cannot calculate the exact cardinality of $|\globsuc_S|$ used above, but we can use the described lower bound instead.
If the local auxiliary concept names have the intended interpretations in all $\Ii_i$, the conditions will be correctly verified.
If not, they will be lower bounds, so $\Theta_i$ might contain less unary types, which only reduces the space of considered models and does not affect the resulting model of $\Kk$; notice that the local auxiliary concept names are introduced by the procedure and are not present in the original KB $\Kk$.

As in the $\ALCOIt$ case, calculate for each individual the set of $r$-reachable nominals for each role name $r$.
For each element of $\Ii_0$, all its direct successors are either inside $\Ii_0$ or are $d_i$ (with unary types $\tau_i$), so verifying whether an element of $\Ii_0$ is in a concept of the form $\qnrleq k r A$ or $\qnrgeq k r A$ can be done by simply counting the appropriate successors.
Concepts of the form $\qnrgeq k {r^*} A$ can be easily verified using the information about the nominals reachable from individuals and the local auxiliary concept names.

\subsubsection{Adjusting the query.}
The only change with respect to the $\ALCOIt$ case regarding the query is the method of calculating whether a path between two individuals exist.
Before we used the sets of nominals $r$-reachable and $r^-$-reachable from the individuals; now we need to use the fact that there are no inverses and the edges between $e_i$ and $d_i$ are always directed towards $d_i$.
Therefore, any (forward) path between individuals is either contained entirely in $\Ii_0$, or it passes through some nominal and can be decomposed into two parts: the beginning up to the last nominal on the path and the rest (possibly empty), contained entirely in $\Ii_0$.
Existence of such parts is easy to verify; the first part relies on the set of nominals $r$-reachable from an individual, which can be deduced from the auxiliary concept names of the form $C_{\qnrgeq 1 {r^*} A_a}$.

\subsection{$\ALCOQt$}
Let $\Kk = (\Tt, \Aa)$ be an $\ALCOQt$ knowledge base with counting threshold $N$, and $Q \in \UCQt$ with $\CQt$s of size at most $m$.

One can decide if $\Kk \fentails Q$ in time
$\poly(\|\Kk\|, (2m)^{|Q|\cdot (\|\Kk\| \cdot N^{|\RCN(\Kk)|})^{O(m)}})$
using oracle calls to finite entailment modulo types over $\ALCQt$ KBs of size at most
$\poly(\|\Kk\|, N^{1+|\RCN(\Kk)|})$,
with trivial ABoxes and unions of at most
$|Q| \cdot (\|\Kk\| \cdot N^{|\RCN(\Kk)|})^{O(m)}$ 
$\CQt$s of size
$O(m)$.

Without loss of generality, assume that $\Kk$ is sticky.

In this case, $\rootint$ becomes $\Ii$ restricted not just to the individuals, but to individuals and all their relevant successors.
For this reason, at the very beginning of the reduction, we do the following.
Iterate through all possible sizes of the set of relevant successors of all individuals (by Fact~\ref{fact:relevant}, it is of size at most $|\Ind(\Kk)|\cdot |\Rol(\Kk)|\cdot N^{|\RCN(\Kk)|}$).
As usual, there is a counter-model iff there is one consistent with one of these choices.
For a fixed choice $t$ of this number, add $t$ fresh individuals to the ABox.
For each role name $r$ and each pair of individuals $(a, b)$, choose whether there exists an $r$-path from $a$ to $b$; if so, add a CI $\{a\} \sqsubseteq \exists r^*. \{b\}$ to the TBox.
Note that this might introduce new nominals, which are treated as such throughout the whole procedure; in particular, appropriate concept names of the form $A_{r, a}$ are introduced.

The construction extends the one from the case without at-most restrictions over transitive closures; below we describe only the newly introduced concepts and verification steps.

\subsubsection{Auxiliary concepts.}
Add auxiliary concept names $C_{\qnrleq k {r^*} A}$ for all $k < N$, role names $r$ and relevant concept names $A$.
Additionally, for each role name $r$ and relevant concept name $A$, choose a subset $\full_{r, A} \subseteq \Delta^{\rootint}$; it is supposed to denote the set of elements having at least $N$ $r$-successors in concept $A$ (one could express it using the concept $C_{\qnrgeq N {r^*} A}$, but we want to avoid increasing the counting threshold in the produced KBs).

\subsubsection{Unravelling.}
Unravelling works the same as in the case of $\ALCOQt$ without at-most restrictions over transitive closures.
Correctness follows easily from the fact that all relevant successors of individuals are inside $\Ii_0$ and that the at-most restrictions are sticky, therefore no additional interaction between the interpretations $\Ii_i$ is introduced.

\subsubsection{Local auxiliary concepts.}
Introduce concept names $D_{\qnrleq k {r^*} A}$ for $k < N$, any role name $r$ and any relevant concept name $A$.

\subsubsection{Verifying being a model of $\Kk$.}
By the assumption that $\rootint$ is closed under taking relevant successors, if any element of $\rootint$ is in the interpretation of a concept name $C_{\qnrleq k {r^*} A}$ for $k > 0$, all appropriate witnesses (i.e. its $r$-successors in concept $A$) need to also be in $\rootint$.
For any fixed role name $r$ and relevant concept name $A$, we verify that if an element of $\rootint$ is not in such concept for any $k$, then it is in the set $\full_{r, A}$.

For all $r$-edges from $e_i$ to $d_i$, if $e_i$ has an at-most restriction over transitive closure of $r$, $d_i$ needs to have the same restriction with number $0$.

 We will again construct the sets $\Theta_i$ of unary types allowed in models of $\Kk_{\tau_i}$: the allowed types $\tau$ need to satisfy all the conditions listed in the case without at-most restrictions over transitive closures. Additionally, if $C_{\qnrleq k {r^*} A} \in \tau$, then: if any nominal \emph{without} the same type of restriction ($C_{\qnrleq \ell {r^*} A}$ for any $\ell$) is $r$-reachable, the restriction is violated (this nominal is in $\full_{r, A}$ and has at least $N$ $r$-successors in concept $A$), so then $\tau$ is forbidden; otherwise, all reachable nominals have the same type of restriction, so all relevant successors reachable through nominals are also nominals and can be counted explicitly.
The upper bound on the number of local witnesses is $\min\{k: D_{\qnrleq k {r^*} A} \in \tau\}$, and, if the local auxiliary concept names have the intended meaning, these numbers are equal.

%% file: appendix-multiple_roles.tex
\section{Eliminating multiple roles}~\label{sec:app:mulroles}

\subsection{Unravelling}
Let $\Imc$ be a finite counter-model.
Let $\Imc_r$ for $r \in \Rol(\Kk)$ be a projection of $\Imc$ to the signature containing only one role name $r$.

We construct a tree-like model $\Jmc$.
We will maintain a homomorphism $h: \Delta^\Jmc \to \Delta^\Imc$ -- each element will be added to $\Jmc$ as a copy of some element of $\Imc$.
Choose any role $r \in \Rol(\Kk)$; the root bag is a copy of $\Imc_r$, identifying the individual with its copy.

Then, recursively, for each added element $x$ proceed as follows:
let $r$ be the role name of the bag in which $x$ was created; for each role name $s \in \Rol(\Kk)\setminus \{r\}$, add a fresh copy of $\Imc_s$ as a child bag, identifying $x$ with its fresh copy in $\Imc_s$.
Repeat the procedure for each created element.

$\Jmc \models \Kk$, as the ABox is satisfied by the root bag and each CI from the TBox concerns only one role.
The homomorphism $h$ from $\Jmc$ to $\Imc$ witnesses that \Jmc does not satisfy $Q$.

\subsection{Proof of Lemma~\ref{lem:multirolesconsistent}}
\subsubsection{The correct $Q$-labelling is consistent.}
Assume the contrary: that a labelling $\Ii'$ is correct and inconsistent.
Inconsistency means that for some bag there exists a subquery $p'$ of $Q$, its partition $p, p_1, p_2, \dots, p_k$ with $\var(p_i)\cap\var(p_j) = \emptyset$ for all $i \neq j$, $V_i = \var(p_i)\cap\var(p)$, $\emptyset \neq V \subseteq \var(p)$, and a match $\eta$ for $p$ in the bag, such that $\eta(V_i) = \{e_i\} \subseteq A_{p_i, V_i}^{\Ii'}$ for all $i$, but $\eta(V) \subseteq \{e\} \not\subseteq A_{p', V}^{\Ii'}$.
As each $e_i \in A_{p_i, V_i}^{\Ii'}$, there is a match $\eta_i$ witnessing this, with $\eta_i(V_i) = \{e_i\}$.
The matchings $\eta, \eta_1, \eta_2, \dots, \eta_k$ are consistent with each other (i.e. each variable of $p'$ is mapped to the same element in all matchings in which it occurs), so they can be combined into a match $\eta'$ for $p'$ in $\Ii'$ (i.e. $\eta' = \eta \cup \bigcup_{i=1}^k \eta_i$) such that $\eta'(V) = \eta(V) \subseteq \{e\}$, thus contradicting the correctness of $\Ii'$.

\subsubsection{If $\Ii$ admits a consistent $Q$-refutation, then the correct $Q$-labelling of $\Ii$ is a $Q$-refutation.}
We will show that if a $Q$-labelling $\Ii'$ is consistent and $e \notin A_{p, V}^{\Ii'}$, then in the correct $Q$-labelling $\Ii''$ also $e \notin A_{p, V}^{\Ii''}$.
Assume the contrary: that $\Ii'$ is a consistent $Q$-labelling, but for some connected fragment $p'$ of $Q$, some $\emptyset \neq V \subseteq \var(p')$, there exists an element $e$ of $\Ii'$ such that $e \notin A_{p', V}^{\Ii'}$, even though there exists a match $\eta$ for $p'$ in $\Ii'$ such that $\eta(V) \subseteq \{e\}$.
From all values of $p', V$ and $e$ for which this happens, choose ones where $p'$ is minimal.
Fix a corresponding match $\eta$ for $p'$.

The image of $\eta$ is not contained inside one bag: assuming it is, for a trivial partition of $p'$ into one part $p$, the consistency condition is clearly violated in this bag -- there exists a match for $p$ in the bag, but $\eta(V) \subseteq \{e\} \not\subseteq A_{p',V}^{\Ii'}$.

Notice that for any binary atom from $p'$, $r(x, y), r^*(x, y)$ or $x = y$, there is a bag to which both $\eta(x)$ and $\eta(y)$ belong; for $r(x, y)$ and $x = y$ this is trivial, and for $r^*(x, y)$ this is true because any two neighbouring bags (i.e. sharing an element) have edges over different roles.
Take any atom from $p'$ mentioning any element of $V$.
Take a bag containing the image of variables mentioned by this atom under $\eta$ (i.e. if it mentions just one variable $x$, take any bag containing $\eta(x)$; if it mentions two variables $x, y$, take any bag containing $\eta(x)$ and $\eta(y)$).
We claim that the compatibility condition in this bag is violated.
Let $p$ be a maximal connected subquery of $p'$ such that $e \in \eta(\var(p))$ and $\eta(\var(p))$ is contained in the chosen bag.
It is easy to see that $p' \setminus p$ seen as a set of connected components $p_1, \dots, p_k$ satisfies the required conditions: the components have disjoint variables by definition; by the definition of the tree structure of the bags (specifically, by condition (1): two bags share a single element if they are neighbours and are disjoint if they are not neighbours), $\eta(\var(p_i)\cap\var(p)) = \{e_i\}$ for some $e_i$; and by minimality of $p'$, $e_i \subseteq A_{p_i,V_i}^{\Ii'}$.

\subsection{Proof of Fact~\ref{fact:emptiness}}


Let $\step_\Bb(X) = \left \{ q \in S \bigm | (X,q) \in
\step_\Bb\right\}$. Consider the mapping $G(X) = X \cup
\step_\Bb(X) $. Because $G$ is monotone and inflationary, for
each $P$ it has the least fixed point $\step^*_\Bb(P) $ that
contains $P$, and the fixed point can be computed by iterating $G$
over $P$.  It is straightforward to see that $\step^*_\Bb(P)$
is the set of states $q$ such that there exists a \emph{finite}
partial run of $A$, whose leaves are labelled with elements of $P$ and
the root is labelled with $q$.  Consider now the mapping $H(X) =
\step^*_\Bb(X \cap F)$. It is monotone and its greatest fixed
point $R$ can be computed by iterating it on $S$. Again, it is not
difficult to see that $\Bb$ accepts some tree iff $I \cap R \neq
\emptyset$. As all the sets and functions above are computable in
polynomial time using the available oracle for $\step_\Bb$,
this gives a polynomial algorithm for non-emptiness.


%% file: appendix-transitive_atoms.tex
\section{Eliminating transitive atoms}

\subsection{Proof of Lemma~\ref{lem:unravelling}}

Let us check that the decomposition is well formed for $\Kk$. We start
the unravelling from the individual mentioned in the ABox, so
$\Ii_\varepsilon \models \Aa$. Item \ref{item:types} of
Definition~\ref{def:well-formed} is satisfied because elements in each
bag inherit their unary types from their originals in $\Jj$.  For item
\ref{item:roots}, note that each neighbour of $\cmp_\Jj(d)$ outside of
$\rel_\Jj(d) \cup \cmp_\Jj(d)$ can have either only outgoing edges to
$\cmp_\Jj(d)$ or only incoming edges from $\cmp_\Jj(d)$: if it had
both, it would belong to $\cmp_\Jj(d)$.  Items \ref{item:external}
follow directly from the construction.  For item
\ref{item:rel-consistent}, note that $\Gamma_u$ and $\Gamma_w$ are
copies of $\rel_\Jj(e_u)$ and $\rel_\Jj(e_w)$ for the originals $e_u$
and $e_w$ of $\rt_u$ and $\rt_w$, which implies that $e_u \in
\rel_\Jj(e_w)$, and this gives $\rel_\Jj(e_u) \subseteq \rel_\Jj(e_w)$
by the definition of relevant successors. It follows by construction
that $\Gamma_u \subseteq \Gamma_w$.  It remains to prove item
\ref{item:rel-contained}. By construction, local elements in bag $u$
are exactly those in the copy of $\cmp_\Jj(e_u)$. By stickyness, all
local elements have the same relevant concept names. Hence, the
concept name $A$ mentioned in item \ref{item:rel-contained} is
relevant for $\rt_u$. Because $d$ mentioned in item
\ref{item:rel-contained} is local, it is reachable from
$\rt_u$. Hence, if $d\in A^{\Ii_u}$, then $d$ is a relevant successor
of $\rt_u$, which means it belongs to $\Gamma_u$. (The only case when
this can happen is when $d=\rt_u$, but this is irrelevant here.)
Consider $v\in T$ such that $\rt_v$ is a sink in $\Ii_u$ or
$\rt_u$ is a non-isolated source in $\Ii_v$. In either case,
the original $e_v$ of $\rt_v$ is reachable from the original $e_u$ of
$\rt_u$ in $\Jj$. Consider an element $e' \in \Gamma_v$. Then, the original $e$
of $e'$ is a (relevant) successor of the original $e_v$ of $\rt_v$,
which implies that it is also reachable from the original $e_u$ of
$\rt_u$. If $e' \in A^{\Ii_v}$, then $e \in A^{\Jj}$ and $e \in
\rel_\Jj(e_u)$. From the construction of the unravelling it follows
that $e' \in \Gamma_u$, and we are done.
If $\Kk$ does not use inverse roles, we add only direct successors of
$\rt_v$, rather than all neighbours, so items \ref{item:roots} and
\ref{item:external} hold with ``sink or source'' replaced with
``sink''.  Because we remove all edges outgoing from
$\Gamma_u\setminus \{\rt_u\}$, these elements are indeed sinks.
If $\Kk$ does not use counting restrictions, $\rel_\Jj(e) = \emptyset$
for each $e \in \Delta^{\Jj}$, so $\Gamma_u = \emptyset$ for each
$u\in T$.

To see that the decomposition is well connected, observe that each
$\Ii_u$ consists of a strongly connected component, a set of elements
attached to it, and possibly a set of isolated elements. This is
precisely how a well-connected interpretation looks.

Safety is proved by contradiction. Consider an infinite path of nodes
$u_0, u_1, \dots, $ such that for all $i$, either $\rt_{u_{i+1}}$ is a
sink in $\Ii_{u_{i}}$ or $\rt_{u_i} $ is a non-isolated source in
$\Ii_{u_{i+1}}$. By construction, in this case, the original
$e_{u_{i+1}}$ of $\rt_{u_{i+1}}$ is a successor of the original
$e_{u_{i}}$ of $\rt_{u_{i}}$ in $\Jj$ for all $i$.  Once a simple
directed path in $\Jj$ leaves some strongly connected 
component (SCC) of $\Jj$ it never goes back. Because the number of
SCCs in $\Jj$ is finite, there exists $i_0$ such that for $i\geq i_0$,
all $e_{u_{i}}$ belong to the same strongly connected component $Z$ in
$\Jj$. By stickiness, the sets $\rel_\Jj(e_{u_{i}})$ coincide for all
$i\geq i_0$; let $\rel_{\Jj}(Z)$ be their common value. From the
construction of $T$ it follows that $\Gamma_{u_i}$ also coincide for
$i\geq i_0$. Because all $\rt_{u_{i}}$ are different, there exists
$i_1 \geq i_0$ such that $\rt_{u_{i}} \notin \Gamma_{u_i}$ for all
$i\geq i_1$. This implies that $e_{u_i} \notin \rel_{\Jj}(Z)$ for
$i\geq i_1$. By the construction of the unravellng, $e_{u_{i+1}}$ is
reachable from $e_{u_i}$ without passing through $\rel_{\Jj}(Z)$, and
$e_{u_{i+1}} \notin \cmp_\Jj(e_{u_i})$. It follows that $e_{u_{j}}
\notin \cmp_\Jj(e_{u_i})$ for $j>i \geq i_1$, because  $e_{u_{j}} \in
\cmp_\Jj(e_{u_i})$ implies that  $e_{u_{j'}} \in \cmp_\Jj(e_{u_i})$
for all $ j \geq j' \geq i$, which is in contradiction with
$e_{u_{i+1}} \notin \cmp_\Jj(e_{u_i})$. But $e_{u_{j}} \notin
\cmp_\Jj(e_{u_i})$ implies $e_{u_{j}} \neq e_{u_i}$ for $j>i \geq
i_1$, which is impossible because $\Jj$ is finite.  

Checking that $\Ii\models \Kk$ is entirely routine.


\subsection{Proof of Lemma~\ref{lem:correct}}

By the compostionality of $Q$-types, that the correct $Q$-labelling
$\sigma$ satisfies the condition
\[  \sigma(d) = \bigg(\tp^{\Ii_u}_Q\big(\{d\} \cup \Gamma_u \cup
\widetilde\Gamma_u\big) \oplus \bigoplus_{e\in \widetilde\Gamma_u}
\sigma(e)\bigg) \upharpoonright \{d\}\cup\Gamma_u
\]
for each node $u \in T$ and each element $d$ local in $u$.

It remains to show that for each $Q$-labelling $\sigma'$ satisfying
the condition above we have $\sigma(d) \subseteq \sigma'(d)$. Let us
fix such $\sigma'$. We first show an auxiliary fact.

\begin{fact} \label{fact:aux}
  For each descendent $v$ of $u$ and each $d$ local in $u$,
  $\sigma'(\rt_v) \upharpoonright \{d\} \cup \Gamma_u \subseteq
  \sigma'(d)$.  
\end{fact}
\begin{proof}
This is proved by induction on the distance between $u$ and $v$. If
$v$ is a child of $u$ the claim follows directly from the initial
assumption on $\sigma'$.  Otherwise, let $w$ be the unique child of
$u$ that is an acestor of $v$. By the inductive hypothesis,
\[\sigma'(\rt_v) \upharpoonright \{\rt_w\} \cup \Gamma_w \subseteq
\sigma'(\rt_w)\,.\] Like in the base case, we get
\[ \big (\sigma'(\rt_v) \upharpoonright \{\rt_w\} \cup \Gamma_w \big)
\upharpoonright \{d\} \cup \Gamma_u \subseteq \sigma'(d)\,. \] The
left-hand side is equal to $ \sigma'(\rt_v)
\upharpoonright \{d\} \cup \Gamma_u$, beacause \[\Delta_v \cap \big(\{d\} \cup
\Gamma_u\big)\subseteq \Delta_v \cap \Delta_u \subseteq \{\rt_w\} \cup
\Gamma_w\] holds in each hybrid decomposition.
\end{proof}

We need to show that for each node $u$, element $d \in \Delta_u$ local
in $u$, and $(p, \eta)$ with $ \emptyset \neq \img(\eta) \subseteq
\{d\} \cup \Gamma_u$, if there exists a match for $p$ in
$\widehat\Ii_u$ that extends $\eta$, then $(p, \eta) \in
\sigma'(d)$. We prove it by induction on the size of a minimal
connected set $V$ of nodes from the subtree rooted at $u$ such that
there exists a match for $p$ in $\bigcup _{v\in V} \Ii_v$, extending
$\eta$.

The base case is when $V = \{v\}$. If $v=u$, then \[(p, \eta) \in
\tp^{\Ii_u}_Q\big(\{d\} \cup \Gamma_u) \subseteq 
\sigma'(d)\,.\] If $v \neq u$, using Fact~\ref{fact:aux},
\begin{align*}
(p, \eta) \in \tp^{\Ii_v}_Q\big(\{\rt_v\} \cup \Gamma_v)
\upharpoonright \{d\} \cup \Gamma_u \subseteq &\\
 \subseteq \sigma'(\rt_v) \upharpoonright \{d\} \cup \Gamma_u
  \subseteq \; & \sigma'(d)\,.
\end{align*}

If $|V| > 1$, let $v$ be the minimal node in $V$. Let us choose a
local element $e$ in $v$ as follows: if $d \in \Ii_v$ take $e=d$,
otherwise take $e=\rt_v$. We will prove that $(p, \eta) \in
\sigma'(e)$. This will yield $(p, \eta) \in \sigma'(d)$ either
directly, or by way of Fact~\ref{fact:aux}. Let $v_1, v_2, \dots, v_k$
be the children of $v$ that belong to $V$. Our goal now is to split
$(p, \eta)$ into $\tau, \tau_1, \dots, \tau_k$ such that
\begin{itemize}
\item
  $\var(p')\setminus \dom(\eta')$ are pairwise disjoint for $(p',
  \eta') \in \tau\cup\tau_1\cup\dots \cup \tau_k$;
\item
  $\img(\eta') \subseteq \{e\}\cup\Gamma_v\cup\widetilde\Gamma_v$ for $(p',
  \eta') \in \tau\cup\tau_1\cup\dots \cup \tau_k$;
\item
  for each $(p',\eta') \in \tau$, an extension of $\eta'$ matches $p'$
  in $\Ii_v$; 
\item
  for each $i$ and each $(p',\eta') \in \tau_i$, an extension of $\eta'$
  matches $p'$ in $\widehat\Ii_{v_i}$; 
\item
  there is a homomorphism $\tilde \eta$, extending $\eta$, from $p$
  into \[\bigcup_{(p',\eta') \in \tau\cup\tau_1\cup\dots \cup \tau_k} \eta'(p')\,;\] 
\end{itemize}
and use the inductive hypothesis. An obstacle is that a path
witnessing a transitive atom $r^*(x,y)$ of $p$ for the match $\widehat
\eta$ may visit multiple bags among $v, v_1, \dots, v_k$. Let us
subdivide each transitive atom $r^*(x,y)$ of $p$ into $r^*(x,z_1),
r^*(z_1,z_2), \dots, r^*(z_\ell,y)$ in a minimal way ensuring that in
the induced matching each atom is witnessed by a path within $\Ii_v$
or within a single $\widehat\Ii_{v_i}$, with endpoints in
$\{\widehat\eta(x), \widehat\eta(y), \rt_v\} \cup \Gamma_v \cup
\widetilde\Gamma_v$.  We can now obtain the desired $\tau, \tau_1,
\dots, \tau_k$ by partitioning the resulting query into maximal
fragments matched within $\Ii_v$ or some $\widehat\Ii_{v_i}$. For each
such fragment $p'$, the corresponding $\eta'$ is obtained by
restricting the matching to variables matched in $\{\rt_v\} \cup
\Gamma_v \cup \widetilde\Gamma_v$.


\subsection{Proof of Lemma~\ref{lem:regularity}}

We now know that $\Kk$ is an $\ALCIt$ KB or an $\ALCQt$ KB, but in the
construction of the automaton we shall only assume that in  well-formed
decompositions for each bag $u$, all elements in $\Gamma_u \setminus
\{\rt_u\}$ are sinks. This is explicitly assumend for $\ALCQt$ KBs in
Definition~\ref{def:well-formed}, and it holds vacuously
for $\ALCIt$ KBs, because $\Gamma_u = \emptyset$.

Let $N$ be the counting threshold in $\Kk$.  Without loss of
generality we can assume that there is an equivalent concept name for
each concept 
\begin{align*}
  &\qnrleq n r A \,, &&\qnrgeq n r A \,, \\
  &\qnrleq n {r^*} B\,, &&\qnrgeq n {r^*} B \,,\\
  &\qnrleq 0 {r^-} A \,, &&\qnrgeq 1 {r^-} A \,, \\
  &\qnrleq 0 {(r^-)^*} A\,, &&\qnrgeq 1 {(r^-)^*} A
\end{align*}
where $n \leq N$, $A \in \CN(\Kk)$, and $B \in \RCN(\Kk)$.  If $\Kk$
does not have this property, we add these concepts and build an
automaton over a richer alphabet. The automaton for the original KB is
obtained by projecting to the smaller alphabet in a natural way.

We first describe the construction of the automaton and then discuss
the additional computability claims. The automaton $\Bb_{\Kk, Q,
\Theta}$ consists of three components. The first one verifies that the
input tree encodes a safe $\Kk$-decomposition of an interpretation
that only uses types from $\Theta$, the second deals with the concept
inclusions of $\Kk$, and the last checks that the decomposition admits
a $Q$-refutation. In what follows, $M=N^{|\RCN(\Kk)|}$ where $N$ is
the counting threshold in $\Kk$ and $[i,j] = \{i, i+1, \dots, j\}$.

\subsubsection*{Correctness of decomposition, safety, types}

The first thing to check is that the input tree is an encoding of a
hybrid decomposition. Correctness of the
encoding amounts to checking that $|\up_u|$ is equal to the length of
the tuple on the edge from $u$ to the parent; the latter is passed
down the tree and compared with $|\up_u|$. Conditions
\ref{item:hybrid-union}--\ref{item:hybrid-disjoint} of
Definition~\ref{def:hybrid} are ensured by the way the decomposition
is represented.

Ensuring that branching is at most $N^{\CN(\Kk)}$ is
done locally, based on the label of the current node and the edges
connecting it to its children, without relying on any kind of in
formation stored in the states.

Let us see how to verify that the hybrid decomposition is well-formed
for $\Kk$. Whether $\Ii_\varepsilon \models \Aa$ is tested locally
based on the label of the current node and recorded in the state; a
state is initial iff the result of this test is positive. Because
fresh nodes in $\Ii_u$ are exactly the ones that do not occur in
$\rt_u$ and $\up_u$, conditions \ref{item:roots} and
\ref{item:external} can also be checked locally. To check
condition \ref{item:types}, it suffices to pass the unary $\Kk$-type
of $\rt_u$ up the tree, and check that it coincides with the unary
$\Kk$-type of the node specified in the label on the edge between $u$
and its parent. To check condition \ref{item:rel-consistent} it is
enough to maintain for each $e \in \Gamma_u$ the information about
which elements of $\Gamma_u$ belong to $\Gamma_v$ where $e=\rt_v$, and 
make sure none of these elements is dropped unless $e$ is
dropped. This involves guessing in each node $w$ if $f_w$ belongs to
$\Gamma_{w'}$, where $w'$ is the parent of $w$: if so, $\up_w$ must
list all elements of $\Gamma_w$; the guess is passed up the tree to be
verified. It remains to check condition \ref{item:rel-contained}.

The first part of condition \ref{item:rel-contained} states that for
each local element $d$ in $\Ii_u$ and each concept name $A$ relevant
for $d$, if $d \in A^{\Ii_u}$ then $d\in\Gamma_u$. This can be be
verified locally by just looking at the label of the current node.
The second part of condition \ref{item:rel-contained} states that for
each local element $d$ in $\Ii_u$ and each concept name $A$ relevant
for $d$, $\Gamma_v \cap A^{\Ii_v} \subseteq \Gamma_u$ for all
$v\in T$ such that $\rt_v$ is a sink in $\Ii_u$ or $\rt_u$ is a
non-isolated source in $\Ii_v$. If $\rt_v \in \Gamma_{u}$, then by 
condition \ref{item:rel-consistent}, the entire $\Gamma_v$ is
contained in $\Gamma_u$. And should $\rt_u$ belong to $\Gamma_v$, it
would be a sink, so not a non-isolated source in $\Ii_v$. Combining
these two observations we see that the automaton 
only needs to verify the second part of condition
\ref{item:rel-contained} for $v$ such that  $\rt_v\notin\Gamma_u$ is a
sink in $\Ii_u$ or $\rt_u \notin \Gamma_v$ is a non-isolated source in
$\Ii_v$, which implies that $v$ is a neighbour of $u$. Towards this
goal, the automaton passes from parent $w'$ to child $w$ the 
information on whether $f_w$ is a sink or a non-isolated source in $\Ii_{w'}$, and
the information about all concept names relevant for some local
elements in $\Ii_{w'}$. The information about relevant concept names
is also passed up the tree: it is guessed in $w'$ and the guess is
passed down in the state to $w$, where it is verified. Based on this
information the automaton can ensure the second part of condition
\ref{item:rel-contained} as follows. If $f_v$ is a sink in the parent
of $v$, the automaton checks that for each concept name $A$ passed to
$v$ from the parent, all elements of $\Gamma_v \cap A^{\Ii_v}$ are
listed in $\up_v$.  Similarly, for each child $u$ of $v$ reachable via
an edge labelled with a non-isolated source in $\Ii_v$, and each concept name
passed from $u$ to $v$, all elements from $\Gamma_v \cap A^{\Ii_v}$
should be listed in the tuple on the edge between $v$ and $u$.

If $\Kk$ does not use inverse roles, the modified conditions
\ref{item:roots} and \ref{item:external} can be checked locally, too,
as can be the additional requirement that all elements of
$\Gamma_u\setminus \{\rt_u\}$ are sinks. The condition that $\Gamma_u
= \emptyset$, imposed when $\Kk$ does not use counting restrictions,
is also local.

To check safety it suffices to ensure that on each infinite branch, after
each node $v$ such that $f_v$ is a sink in the parent of $v$, there
is a node $w$ such that $f_w$ is a non-isolated source in the parent of $w$. The
automaton already has access, in node $v$, to the information on
whether $f_v$ is a sink or a non-isolated source in the parent $v'$ of
$v$. Additionally, we shall maintain the same information about the
parent $v'$ of $v$. With that, it suffices to declare as accepting
those states of the automaton, where $f_v$ is a sink and $f_{v'}$ is a
non-isolated source, or the other way around.

Verifying that only unary types from $\Theta$ are realized is also
done locally; when an element realizing a type not in $\Theta$ is
detected, the automaton rejects immediately. 

Including additional information that will be useful later on, we let
the state $q$ at the node $v$ represent the following information:
\begin{enumerate}
\item \label{item:aut:rt}
  a type $\tau_q \in \Tp(\Kk)$, storing the unary $\Kk$-type of $f_v$;
\item \label{item:aut:Gamma}
  a tuple $(\tau^1_q, \dots, \tau^k_q) \in \Tp(\Kk)^k$ for some $k\leq
  M$, storing the unary $\Kk$-types of elements of $\Gamma_v$
  (according to some arbitrary fixed order on $\Delta$);
\item\label{item:aut:rt-in-Gamma}
  a value $i_q \in [0, k]$, where $i_q \in [1, k]$
  indicates that $f_v$ is the $i_q$th element in $\Gamma_v$ and
  $i_q=0$ means that $f_v \notin \Gamma_v$;
\item \label{item:aut:map}
  a partial function $\iota_q$ from $[1,k]$ to $[1, M]$, indicating
  which elements of $\Gamma_v$ are listed in $\up_v$ and their
  positions in the tuple of $\Kk$-types in the parent (this also
  provides $|\up_v|$); 
\item \label{item:aut:luggage}
  a binary reflexive relation $\leadsto_q$ over $[1,k]$ representing
  for each element in $\Gamma_v$, a subset of $\Gamma_v$ containing
  this element;
\item \label{item:aut:relevant}
  a subset $R_q$ of $\RCN(\Kk)$ indicating concept names relevant for
  local elements in $v$;
\item \label{item:aut:relevant-parent}
  a subset $R'_q$ of $\RCN(\Kk)$ indicating concept names relevant for
  local elements in the parent $v'$ of $v$;
\item \label{item:aut:parent}
  a binary flag $\mathsf{sink}_q$ indicating whether $f_v$ is a sink
  or a source in the parent $v'$ of $v$;
\item \label{item:aut:grandparent}
  a binary flag $\mathsf{sink}'_q$ indicating whether
  $f_{v'}$ is a sink or a source in  the parent $v''$ of $v'$;
\item a binary flag $\mathsf{abox}_q$ indicating whether $\Ii_v \models \Aa$.
\end{enumerate}
Moreover, we only allow combinations of these values that satisfy the
following consistency conditions:

\begin{itemize}
\item if $i_q \neq 0$, then $\tau_q = \tau_q^{i_q}$ and $i_q
  \leadsto_q j$ for all $j \in [1,k]$; 
\item if $i \in \dom(\iota_q)$ and $i \leadsto_q j$, then $j \in 
  \dom(\iota_q)$;
\item if $\mathsf{sink}_q = 1$, then for each $A\in R'_q$ and  $j \in
  [1,k]$  it holds that $A \in \tau_q^j$ implies $j \in \dom(\iota_q)$.
\end{itemize}

\subsubsection{KB automaton}

Concept inclusions that involve only concept names are verified
locally. For the remaining checks, apart from the information stored
in the first component, it suffices to pass some aggregated
information about each element that is shared between nodes; that is,
$f_v$ and $\Gamma_v$ for each $v$. In what follows, by an $A$-successor
we mean a successor that belongs to the extension of the concept
$A$. Similarly for predecessors, direct successors, and direct
predecessors. The second component of state $q$ at node $v$ stores:
\begin{enumerate}
\item  a subset of $\dom(\iota_q)$, representing elements of $\Gamma_v$
  shared with the parent of $v$
  that are \emph{successors} of $\rt_v$ in the current subtree;
\item  a subset of $\dom(\iota_q)$, representing elements of $\Gamma_v$
  shared with the parent of $v$ 
  that are \emph{direct successors} of $\rt_v$ in the current subtree;
\item for each $A\in \CN(\Kk)$, the counts of
\begin{itemize}
\item  $A$-predecessors and direct $A$-predecessors (up to $1$),
\item  $A$-successors and direct $A$-successors (up to $N$)
\end{itemize}
of $\rt_v$ appearing in the current subtree as local elements;
\item for each $i\in [1,k]$ and $A\in \CN(\Kk)$, the counts of
\begin{itemize}
\item  $A$-successors and direct $A$-successors (up to $N$)
\end{itemize}
of the $i$th element of $\Gamma_v$ appearing as local elements in the
subtree rooted at the node $w$ such that $\rt_w$ is the $i$th element
of $\Gamma_v$. 
\end{enumerate}
If $i_q \neq 0$, then the counts for the $i_q$th element of
$\Gamma_v$ must be consistent with those for $\rt_v$.

It is not difficult to maintain this information when processing the
tree, detecting violations of $\Kk$ along the way. The latter is done
when an element is \emph{forgotten}; that is, its record is not passed up
the tree any more. This happens exactly in the node where the said
element is fresh. Assuming that the current node is $v$, this
includes elements internal in $\Ii_v$ and $\rt_w$ for each child $w$ of
$v$ (among the latter, all elements of $\Gamma_v$ that are not listed in
$\up_v$). For each such element we compute the current counts,
based on $\Ii_v$ and the information passed from the children and the
parent (including the information on elements of $\Gamma_v$ that are
reachable from each $\rt_v$ and from other elements of $\Gamma_v$):
because all counts refer to local elements in disjoint parts of the
tree, there is no danger of counting anything twice. If the computed
counts violate some restriction imposed by $\Kk$, the automaton
rejects immediately. The maintanance amounts to updating the records
for those elements $e$ that are not getting forgotten in the current node. Note
that this requires accounting for those elements of $\Gamma_v$ that
are reachable from $e$ and are getting forgotten in $v$.
\todo{Finish maintanance.}

Observe that correctness of this construction relies on the safety condition.
Witnesses for positive counts of $A$-successors or $A$-predecessors
must be found among elements reachable or backwards reachable from the 
elements in the current bag. Safety implies that these elements are
contained within a finite fragment of the decomposition.

\subsubsection{Query automaton}

The automaton guesses a $Q$-refutation and verifies that it is
consistent. The consistency condition, as given in
Definition~\ref{def:consistency}, can be easily checked based on the
label of the current node and the labels of the edges connecting it
with its children, provided that we are given the values
$\sigma(\rt_v)$ for all children $v$ of the current node. In fact,
note that in the condition of Definition~\ref{def:consistency}, one
can replace $\sigma(e)$ with $\sigma(e) \upharpoonright \{e\} \cup
\Gamma_u$. This is precisely the information the automaton will pass
between nodes: states of the query automaton are $Q$-types with
parameters $[0, M]$, where $0$ represents $\rt_v$ and each $i > 0$
represents $\up_v^i$. Moreover, we only allow $Q$-types in which no $q 
\in Q$ occurs.

\subsubsection{Computability}

The set of states, initial states, and accepting states of the product
automaton $\Bb_{\Kk, Q, \Theta}$ can be clearly computed within the
desired complexity bounds. Let us see how to deal with
$\step_{\Bb_{\Kk, Q, \Theta}}$. Let $P$ be a set of states of
$\Bb_{\Kk, Q, \Theta}$ and $q$ a single state of $\Bb_{\Kk, Q,
  \Theta}$. As a first step we remove from $P$ all states that are
incompatible with $q$. That is, we only keep states $p$ such that 
\begin{itemize}
\item $\tau^i_p  =  \tau^{\iota_p(i)}_q$ for all $i \in \dom(\iota_p)$;
\item $R'_p = R_q$;
\item $\mathsf{sink}'_p = \mathsf{sink}_q$;
\item if $\mathsf{sink}_p = 0$, then $A \notin   \tau^i_q$ for all $A \in
  R_p$, $i \notin   \img(\iota_p)$;
\item $\left \{\iota_p(j) \bigm | i \leadsto_p
    j\right\} = \left \{j' \bigm | \iota_p(i) \leadsto_q
    j'\right\}$ for $i\in \dom(\iota_p)$;
\item for all $i\in \dom(\iota_p)$
  and all $A \in \CN(\Kk)$, 
    the counts of $A$-successors and direct $A$-successors for
    $i$ in $p$ are equal to the respective counts for $\iota_p(i)$ in $q$.
\end{itemize}

Now, we shall construct a single-role $\ALCOIQfwdt$ KB $\Kk'$ without
at-most restrictions over closures of roles, a query $Q' \in\UCQ$, and
a type set $\Theta'$ such that $(P, q) \in \step_{\Bb_{\Kk,Q,\Theta}}$
iff $\Kk' \notfentails^{\Theta'} Q'$. For $\Kk$ in $\ALCIt$, the
construction almost directly gives $\Kk'$ in $\ALCIt$: nominals are
not needed because the sets $\Gamma_u$ are empty in
$\Kk$-decompositions, and CIs using counting will not be added.
If $\Kk$ is a $\ALCQt$ KB, $\Kk'$ will be in $\ALCOQt$. Nominals can
be eliminated as explained in Appendix~\ref{sec:ABoxesNom}, but
because the ABox is already trivial, there is no need of unravelling
and the reduction will give only one instance of $\ALCQt$.

\newcommand{\rootC}{\mathsf{Root}}
\newcommand{\sink}{\mathsf{Sink}}
\newcommand{\source}{\mathsf{Source}}
\newcommand{\internal}{\mathsf{Internal}}
\newcommand{\local}{\mathsf{Local}}
\newcommand{\rels}{\mathsf{Gamma}}
\newcommand{\external}{\mathsf{External}}

Let us start with $\Kk'$. The ABox mentions a single individual $a_0$
and specifies its unary type as $\tau_q$. The TBox uses nominals $a_1,
\dots, a_k$ and specifies the unary type of $a_i$ as $\tau^i_q$. If
$i_q = 0$, then $a_0 \notin \{a_1, \dots, a_k\}$; otherwise, $a_0 =
a_{i_q}$.

We include into $\Kk'$ the axioms
\begin{align*}
  \rootC & \equiv \{a_0\}\,,\\
  \rels & \equiv \{a_1,  \dots, a_k\}\,,\\
  \sink & \equiv \forall r. \bot \,,\\
  \source & \equiv \forall r^-. \bot  \sqcap \exists r.\top\,,\\
  \internal  & \equiv \lnot(\sink \sqcup \source) \,,\\
  \local  & \equiv \internal \sqcup \rootC \,,\\
  \rels & \sqsubseteq \sink \sqcup \rootC 
\end{align*}
for fresh concept names $\rootC$, $\rels$,
$\sink$, $\source$, $\internal$, and $\local$.
(If $\Kk$ does not use inverse roles, let $\source \equiv \bot$.)
For each CI $A \sqsubseteq \qnrleq n {r^*} B$ in $\Kk$ with $n>0$, if $B \in R_q$, then add
\begin{align*}
    A  \sqcap B \sqcap \local & \sqsubseteq  \rels \,,
\end{align*}
and
otherwise add 
\begin{align*}
  A  \sqcap \local & \sqsubseteq  \bot \,.
\end{align*}
Using these, we reformulate the CIs of $\Kk$ taking into account that some
neighbours of non-internal elements are outside of the current bag.
Aiming at reducing the number of cases, we model $\exists s. B$ as
$\qnrgeq 1 s B$ and $\forall s. B$ as $\qnrleq 0 s {\overline B}$ for
$s \in \big\{r^-, (r^-)^*\big\}$.
%

To model neighbours inside and outside the current bag, for each $B
\in \CN(\Kk) \cup \{ \{a_i\} \bigm| 0\leq i \leq k\}$  and $n\leq N$ we introduce fresh concept names
\[E_{\qnrleq n r B} \,, \quad  E_{\qnrgeq n r B}\,, \quad E_{\qnrleq n
    {r^*} B} \,, \quad  E_{\qnrgeq n {r^*} B}\,, \]
\[E_{\qnrleq 0 {r^-} B} \,, \;\;  E_{\qnrgeq 1 {r^-} B}\,, \;\; E_{\qnrleq 0 {(r^-)^*} B} \,, \;\;  E_{\qnrgeq 1 {(r^-)^*} B} \,.\]
Let  $s \in\{r, r^-\}$. 
For each CI of the form $A \sqsubseteq \qnrgeq n s B$ in $\Kk$, we add
\begin{align*}
  A &\sqsubseteq \bigsqcup_{i \leq n}   {\qnrgeq i s B} \sqcap
      E_{\qnrgeq {n-i} s B} \,,
\end{align*}
and analogously for CIs of the form $A \sqsubseteq \qnrleq n s B$.
For each CI of the form $A \sqsubseteq \qnrgeq 1 {s^*} B$, we
add 
\begin{align*}
  %
  A \sqsubseteq  \qnrgeq 1 {s^*} {E_{\qnrgeq 1 {s^*} B}} \,,
\end{align*}
and for each CI of the form $A \sqsubseteq \qnrleq {0} {s^*} B$, we add
\begin{align*}
  A &
      \sqsubseteq \qnrleq 0 {s^*} {E_{\qnrgeq1 {s^*} {B}}} \,.
\end{align*}
For each CI of the form $A \sqsubseteq \qnrgeq {n} {r^*} B$,
we add 
\begin{align*}
  %
  A \sqsubseteq \;
  \bigsqcup_{\alpha_0 + \sum_{i=1}^n \alpha_i\cdot i = n} \;
  F_{\qnrgeq {\alpha_0} {r^*} B} \sqcap
  \bigsqcap_{i=1}^n  F_{\qnrgeq {(\alpha_i, i)} {r^*} B}
  \,,\\
 \hspace{-7ex}   F_{\qnrgeq {\alpha} {r^*} B} \sqsubseteq\;
  \bigsqcup _{I \subseteq [1,k]\colon \alpha \leq \sum_{i\in I}
  \ell^B_i} \;\bigsqcap_{i \in I} F_{\qnrgeq 1 {r^*} {\{a_i\}}} 
  \,,\\
  F_{\qnrgeq 1 {r^*} {\{a_i\}}} \sqsubseteq
  \qnrgeq 1 {r^*} {E_{\qnrgeq 1 {r^*} {\{a_i\}}}}
  \,,\\
  F_{\qnrgeq {(\alpha,\beta)} {r^*} B} \sqsubseteq
  \qnrgeq {\alpha} {r^*} {\lnot \rels \sqcap E_{\qnreq {\beta}{r^*} B}}
  \,,\\
  E_{\qnreq {\beta} {r^*} B}
  \sqsubseteq   E_{\qnrleq {\beta} {r^*} B} \sqcap E_{\qnrgeq {\beta} {r^*} B}
\end{align*}
for $\alpha,\beta \leq N$, $i\in [1,k]$, where $\ell_i^B$ is the count of
$B$-successors of $a_i$ stored in $q$.
For each CI of the form $A \sqsubseteq
\qnrleq {n} {r^*} B$ where $n>0$ and $B \in R_q$, we add
\begin{align*}
  \local \sqcap A &\sqsubseteq
  \bigsqcup_{O \subseteq \widehat B \colon |O| \leq  n} \;
  \bigsqcap_{o \in \widehat B \setminus O}
  F_{\qnrleq 0 {r^*} {\{o\}}} \,,\\
  F_{\qnrleq 0 {r^*} {\{a_i\}}} &
                                \sqsubseteq \qnrleq 0 {r^*} { E_{\qnrgeq 1 {r^*} {\{a_i\}}}} \,,
\end{align*}
where $\widehat B = \left\{a_i \bigm | B \in \tau_q^i\right\}$.

We need to ensure that the concepts modelling neighbours are
interpreted in a way consistent with the information in the
states. First, we ensure that internal elements do not have any
external neighbours by adding
\begin{align*}
  \internal &
                     \sqsubseteq
                     \bigsqcap_{n \geq  0} E_{\qnrleq n s B} \,,\\  
  \internal &
                     \sqsubseteq
                     E_{\qnrgeq 0 s B} \sqcap
                     \bigsqcap_{n > 0} \lnot  E_{\qnrgeq n s B} \,,\\  
  \lnot B \sqcap \internal &
                     \sqsubseteq
                     \bigsqcap_{n \geq  0} E_{\qnrleq n {s^*} B} \,,\\  
  \lnot B \sqcap \internal &
                     \sqsubseteq
                     E_{\qnrgeq 0 {s^*} B} \sqcap
                     \bigsqcap_{n > 0} \lnot  E_{\qnrgeq n {s^*} B} \,,\\  
  B \sqcap \internal &
                     \sqsubseteq
                       \lnot E_{\qnrleq 0 {s^*} B} \sqcap
                       \bigsqcap_{n >  0} E_{\qnrleq n {s^*} B} \,,\\  
  B \sqcap \internal &
                     \sqsubseteq
                       \bigsqcap_{n \leq 1}  E_{\qnrgeq n {s^*} B} \sqcap
                       \bigsqcap_{n > 1} \lnot  E_{\qnrgeq n {s^*} B}
\end{align*}
for all $B \in \CN(\Kk) \cup \{ \{a_i\} \bigm| 0\leq i \leq k\}$ and $s \in \left\{r, r^-\right\}$.

Next, we take care of sinks. For those not in $\{a_0, a_1\dots, a_k\}$ we let
\begin{align*}
\sink \sqcap \lnot (\rels\sqcup \rootC)\sqsubseteq \;\bigsqcup_{\substack{p
  \in P \colon i_p \notin \dom(\iota_p)\\ \mathsf{sink}_p=1}} \;\bigsqcap_{A \in \widehat\tau_p}\; A \,,
\end{align*}
where $\widehat \tau_p$ is the extension of $\tau_p$ to concept names
$E_{\qnrleq n t B}$ and $E_{\qnrgeq n t B}$ that reflects the counts
of neighbours in the current subtree stored in $p$.
For $a_i$ with $i \notin \dom(\iota_q)$ the condition is similar:
\begin{align*}
\{a_i\} \sqsubseteq \;\bigsqcup_{p
  \in P \colon \iota_p(i_p) =i} \;\bigsqcap_{A \in \widehat\tau_p}\; A \,.
\end{align*}
\todo{Finish backwards requirements for local nominals. }
For $a_i$ with $i \in \dom(\iota_q)$, different from $a_0$, the
condition is extracted exclusively from the counts stored in $q$; an
appropriate subtree will be provided in the bag where the element
corresponding to $a_i$ is fresh. 

For sources different from $a_0$, shared elements might be successors or direct
successors both in the current bag and outside. To handle that, for
each conjunction $C$ choosing either $\qnrleq 0 t {\{ a_i\}}$ or $\qnrgeq
1 t {\{ a_i\}}$ for each $t\in\{r, r^*\}$ and $i\in [1,k]$, we add a
separate conditions
\begin{align*}
C \sqcap \source \sqcap \lnot \rootC \sqsubseteq \;\bigsqcup_{\substack{p
  \in P \colon i_p \notin \dom(\iota_p)\\ \mathsf{sink}_p = 0}} \;\bigsqcap_{A \in \widehat\tau_p^C}\; A \,,
\end{align*}
in which $\widehat \tau_p^C$ accounts for the connections to $a_i$ with
$i \in \img(\iota_p)$ described in $C$.

Finally, we let
\begin{align*}
\{a_0\} \sqsubseteq \bigsqcap_{A \in \widehat\tau_q}\; A \,,
\end{align*}
where $\widehat \tau_q$ extracts information about counts of external
neighbours from the counts in the current subtree stored in $q$ and
the total counts represented in the $\Kk$-type $\tau_q$.

Note that many disjunctions above have doubly exponential length. To avoid
blowup, instead of actually putting them into $\Kk'$ we encode them in the
type set $\Theta'$, by filtering out types that do not satisfy
them. Of course, $\Theta'$ also filters out types that are not
extensions of types from $\Theta$. 

Let us deal with the query now. 


\begin{definition}[unary $Q$-type]
  A unary $Q$-type with parameters $\Gamma$ is a set of triples of the
  form $(p, V, \eta)$ such that $p$ is a fragment of $Q$,
  $V \subseteq \var(p)$, and $\eta$ is a partial mapping from $\var(p)
  \setminus V$ to $\Gamma$.
\end{definition}

The definition below captures the condition relating the the $Q$-type
stored in state $q$ with $Q$-types stored in states from $P$. The
function $\Sigma$ it mentions in our case maps each 
$\tau \in \Tp(\Kk')$ to the set $\Sigma(\tau)$ of unary $Q$-types
 with parameters $\{a_1, a_2,
\dots, a_k\}$, 
\[\left\{ \big( f, \eta^{-1}(\{0\}),
    a_*\circ \iota_p \circ \eta\upharpoonright [1,k] \setminus \{i_p\}
    \big) \bigm | (f, \eta) \in \tau^Q_p \right\},\]
where $a_*$ is a function mapping $i$ to $a_i$, $\tau^Q_p$ is the
$Q$-type the stored in state $p$, and $p$ ranges over states in $ P$ that are compatible
with $\tau$; that is, satisfy one of the following conditions: 
\begin{itemize}
\item $\left\{\sink, \overline{\rels}, \overline{\rootC}\right\} \cup \widehat
  \tau_p\subseteq \tau$, $\mathsf{sink}_p \!=\! 1$, $i_p \!\notin\!
  \dom(\iota_p)$,
\item $\left\{\source,  \overline{\rootC}\right\} \cup \widehat \tau_p^C
  \cup C \subseteq \tau$, $\mathsf{sink}_p = 0$, $i_p \notin
  \dom(\iota_p)$,
\item  $\left\{\{a_i\}\right\} \cup \tau_p \subseteq \tau$,
  $\iota_p(i_p) =i \notin \dom(\iota_q)$. 
\end{itemize}

\begin{definition} [weak realizability modulo]
  Consider a KB $\Kk'$, a $Q$-type
$\tau$ with parameters $\Gamma \subseteq \Ind(\Kk')$, and a function
$\Sigma$ mapping unary $\Kk'$-types to sets of unary $Q$-types with
parameters $\Gamma$. The $Q$-type $\tau$ is \emph{weakly realized
modulo $\Sigma$} in a $\Kk'$-interpretation $\Mm$ if
  \[\tau \supseteq \bigg(\wtp^\Mm_Q\big(\Gamma\cup\dom(\sigma)\big)
    \oplus \bigoplus_{d\in \dom(\sigma)} \sigma(d) \bigg)
    \upharpoonright \Gamma\,.\]
  for some partial function $\sigma$ mapping elements $d\in\Delta^\Mm$ to $Q$-types
  $\sigma(d)$ with parameters $\{d\} \cup \Gamma$ such that 
  $ \left \{ \left(p, \eta^{-1}(d), \eta\upharpoonright \Gamma\right) \bigm |
    (p,\eta)\in\sigma(d) \right\} \in \Sigma\big(\tp^{\Kk'}_\Mm(d)\big)$.
\end{definition}

\begin{lemma}
  Given $Q\in\UCQt$, a KB $\Kk'$, a set $\Theta' \subseteq \Tp(\Kk')$,
  a $Q$-type $\tau$ with parameters
  $\Gamma \subseteq \Ind(\Kk')$, and a function $\Sigma$ mapping unary
  $\Kk'$-types to sets of unary $Q$-types with parameters $\Gamma$,
  one can construct $Q'\in \UCQ$ and $\Theta''$ such that $\Kk'
  \fentails^{\Theta''} Q'$ if $\tau$ is weakly realizable modulo $\Sigma$ in a 
  finite model of $\Kk'$ realizing only types from $\Theta'$.   
\end{lemma}

\begin{proof}
For each triple of the form $(p, V, \eta)$ such that $p$ is a
subquery of a query from $Q$, $V \subseteq \var(p)$, and $\eta$ is a
partial mapping from $\var(p) \setminus V$ to $\Gamma$, introduce a
concept $A_{p, V, \eta}$. This adds $|Q|\cdot 2^m\cdot (|\Gamma|+1)^n$
concepts, where $m = \max_{q\in Q}|q|$ and $n = \max_{q\in Q}
|\var(q)|$.

Let $\Theta''$ be the set of all unary types (over the
extended set of concept names) of the form $\theta \cup \theta'$ where
$\theta \in \Theta'$ and there exists a unary $Q$-type $\upsilon \in
\Sigma(\theta)$ such that $\overline{A_{p, V, \eta}} \in \theta'$ for
all $(p, V, \eta) \in \upsilon$. In the model we seek, we shall
allow only unary types from $\Theta''$.

We now define a query $Q' \in \UCQ$ that will be forbidden in
the model we seek. Choose
\begin{itemize}
\item
  a fragment $p$ of $Q$ and a partial mapping from
  $\var(p)$ to $\Gamma$ such that $(p, \eta) \notin \tau$;
\item
  a partition of $p$ into fragments $p'$ and $p_1, \dots, p_k$ such that
  $\var(p_i) \cap \var(p_j) \subseteq \var(p') \cup \eta^{-1}(\Gamma)$
  for all $i \neq j$;
\item a localization $p''$ of $p'$.
\end{itemize}
For each choice, add to $Q'$ the query
$\eta(p'' \land \bigwedge_{i=1}^k \widetilde p_i)$, where
\begin{align*}
  \widetilde p_i &= A_{p_i, V_i, \eta_i}(y_i) \land \bigwedge_{x \in
                   V_i} x = y_i\,,\\
  \eta_i &= \eta\upharpoonright \var(p_i)\,,\\
  V_i &= \var(p_i) \cap \var(p') \setminus \eta^{-1}(\Gamma)\,,
\end{align*}
and $y_i$ is a fresh variable. With $m = \max_{q \in Q} |q|$,  we have
\[|Q'| \leq
  \big(|Q|\cdot 2^m \cdot (|\Gamma|+1)^n\big) \cdot m^m \cdot  2^m
\]
and, after eliminating equality atoms in the usual, way
each conjunctive query in $Q'$ has size $O(m)$.
\end{proof}

\subsection{Proof of Lemma~\ref{lem:folding-correct}}

Let us call an interpretion \emph{$\ell$-bounded} if the length of
simple directed paths in this interpretation is at most $\ell$.

\begin{lemma} \label{lem:bounded-transfer}
  $\Ff$ is $\ell$-bounded.
\end{lemma}

\begin{proof}
  Redirected edges orginate in $\Ii_u$ such that $\rt_u$ is a source
  in $\Ii_{u'}$ for the parent $u'$ of $u$, and lead to $\rt_{\hat v}$
  that is a sink in $\Ii_{\hat v'}$ for the parent $\hat v'$ of $\hat
  v$. Consequently, a simple directed path $\pi$ taking an edge 
  redirected from $\rt_v$ to $\rt_{\hat v}$ (as the first redirected
  edge) cannot reach the origin of another redirected edge. Hence, the 
  suffix of $\pi$ starting in $\rt_{\hat v}$ is a path in $\Ii'$, and
  its length is bounded by $\ell$. Because $N^{\Ii_n}_n(\rt_{v})\simeq
  N^{\Ii_n}_n(\rt_{\hat v})$ and $n\geq \ell$, a simple directed path
  originating in $\rt_{v}$, of the same length as the suffix, can be
  found in $\Ii'$. Consequently, there exists a simple path in
  $\Ii'$ of the same length as $\pi$: instead of taking the
  redirected edge, take the original edge to $\rt_v$ and continue from
  there; because $\pi$ enters $\rt_v$ from $\Ii_{v'}$ for the parent
  $v'$ of $v$, and $\rt_v$ is a sink in $\Ii_{v'}$, the resulting path
  is indeed simple. It follows that the length of the whole $\pi$ is
  bounded by $\ell$. 
\end{proof}

\begin{lemma}
  $\Ff \models \Kk$.
\end{lemma}
\begin{proof}
  Because the unary types of $\rt_v$ and $\rt_{\hat v}$ coincide,
  redirections of edges do not violate concept inclusions of $\Kk$
  that do not involve transitive closure. In $\ell$-bounded
  interpretations, each universal restriction is equivalent to not
  satisfying a UCQ with at most $\ell$ binary atoms per
  CQ. Consequently, $\Ff$ satisfies all universal restrictions by 
  Fact~\ref{fact:coloured-blocking}.  It remains to verify
  existential restrictions involving transitive closure. Consider a
  forward  path $\pi$ witnessing an existential restriction in $\Ii'$
  and suppose that the origin of this path is still in still in
  $\Ff$. Because $\Ii'$ is $\ell$-bounded, the path has length at
  most $\ell$. If this path does not pass through a redirected edge, we
  are done. Suppose that it does pass an edge redirected from $\rt_v$
  to $\rt_{\hat v}$. Then we can transfer the suffix of the path
  after $\rt_v$ to $\rt_{\hat v}$, because  $N^{\Ii_n}_n(\rt_{v})\simeq
  N^{\Ii_n}_n(\rt_{\hat v})$ and $n\geq \ell$. Like in the proof of
  Lemma~\ref{lem:bounded-transfer}, the transferred suffix cannot take 
  another redirected edge, and we are done.  
\end{proof}

%% file: hardness.tex

\section{Hardness}
\par \noindent {\bf Theorem~\ref{thm:LoweBound}}
Finite entailment of  $\CQt\!$s over \ALC knowledge bases is \twoexp-hard. 
\newcommand{\Kmw}{\ensuremath{\Kmc_{(\Mmc, w)}}}
\newcommand{\qmw}{\ensuremath{q_{\Mmc,w}}}

\begin{proof}
We reduce the word problem of exponentially space bounded alternating Turing machines. 
For each input word $w$ to $\Mmc= (Q, \Sigma, q_0, \delta)$, we define an \ALC knowledge base $\Kmw$ and a $\CQt\!$ $q$ such that $\Mmc$ accepts $w$ iff there is a counter-model of $\Kmw$ and $\qmw$. 

Each finite forest counter-model of $\Kmw$ and $\qmw$ will represent an accepting computation of $\Mmc$ and $w$. 
Our encoding follows closely the one use in~\cite{EiterLOS09} to show that (arbitrary) entailment of CQs
over $\mathcal{SH}$ knowledge bases is \twoexp-hard. In particular, each tree model of $\Kmw$ encodes an accepting \emph{computation tree} in which each node is the root of a \emph{configuration tree} (see Figure~\ref{fig:comp-tree}). Notably, there are two consecutive edges connecting a configuration tree and any of its successor configurations. Further, each configuration tree is a complete binary tree of depth  equal to the length of the input word $|w|$, and whose leaves shall store the tape contents using a \emph{cell gadget} as illustrated in Figure~\ref{fig:conf-tree}. Each \emph{cell gadget} 
records the current (right descendants) and the previous (left descendants) configuration of $\Mmc$.  The particular structure of a cell gadget is illustrated in Figure~\ref{fig:cell-gadget}.  The root of each gadget cell is labelled with concepts names $B_1, \dots, B_m$ that encode (in binary) the position of the cell in the tape.  The node labelled with concept name $E_h$ records pairs $(q,a)$ (not) satisfied in the current configuration, whereas $F_h$ records all possible pairs $(q,a)$ with $q \in Q$ and $a \in \Sigma$ using concept names $Z_{q,a}$. The actual content of the cell in the current (previous) configuration is recorded in the node labelled by $G_h$ ($G_p$) using a concept name for each symbol in $\Sigma$; 
symmetrically, the node $F_p$ records the pairs  (not) satisfied in the previous configuration and $E_p$ all possible pairs $(q,a)$. 
 
Concept inclusions in $\Kmw$ to enforce the structure fo configuration trees are rather standard and use a single role $r$. They ensure locally  that in every gadget the current configuration is indeed a successor of the previous configuration. All the relevant axioms to ensure this from the encoding in ~\cite{EiterLOS09} can be expressed in \ALC. 

\medskip 
We will use a $\CQt\!$   $q_w$  to ensure that computation trees are \emph{proper}, i.e., that the previous configuration recorded in each configuration tree $T$ coincides with the current configuration recorded in all its successors $T'$. This property can be characterised as follows:

\begin{lemma}[Proposition 4 in~\cite{EiterLOS09}]\label{lemma:ProperLB}
A computation tree is not proper iff  there exists a cell gadget $n$ in some configuration tree $T$ and a cell gadget $n'$ in a successor configuration of $T$ such that for all $A \in \Bbf \cup \Zbf$ 

\begin{itemize}
\item[$(\dagger)$] $A$ is true at the $E_h$-node of $n$ and the $E_p$-node of $n'$, or
\item[$(\ddagger)$]  $A$ is true at the $F_h$-node of $n$ and the $F_p$-node of $n$.
\end{itemize}
\end{lemma}
where $\Bbf = \{B_1, \dots, B_m\}$ is the set of all concept names used to encode the cell addresses, and $\Zbf = \{ Z_{q,a} \mid q \in  Q, a\in \Sigma \}$
is the set of concept names used as markers in the cell gadgets. 
Notably, (the proof of) Lemma~\ref{lemma:ProperLB} does not rely on the use of role inclusions not transitive roles in the encoding. 

Thus, the query $q_w$ should be able to test properties $(\dagger)$ and $(\ddagger)$ for each $A \in \Bbf \cup \Zbf$. We achieve this by taking a copy of the query shown in Figure~\ref{fig:BasicQ}, for each $A  \in \Bbf \cup \Zbf$.  Possible matches of one of these copies in a computation tree are illustrated on Figure~\ref{fig:BasicQ}. The required query $q_w$ is the conjunction of all these copies.

\end{proof}

\begin{figure}\caption{ATM computations}\label{fig:comp-tree}
\centerline{%
\begin{tikzpicture}[>=latex,point/.style={circle,draw=black,minimum size=1.4mm,inner sep=0pt}]
\node[point, fill=black] (t1) at (0,2){};
\node[point, fill=black] (t12) at (1.5,1.5){};
\node[point, fill=black] (t22) at (3,1){};
\node[point, fill=black] (t13) at (4.5,0.5){};
\node[point, fill=black] (t23) at (6,0){};
\node[point,fill=black] (t14) at (4.5,2){};
\node[point, fill=black] (t24) at (6,3){};
\draw[thin] (0,2) -- (-1,0) -- (1,0) -- cycle;
\draw[->,out=10,in=130](t1) to node[pos=0.5,sloped, above] {$r$}  (t12); 
\draw[->,out=10,in=130](t12) to node[pos=0.5,sloped, above] {$r$}  (t22);
\draw[->,out=10,in=130](t22) to node[pos=0.5,sloped, above] {$r$}  (t13);
\draw[->,out=10,in=130](t13) to node[pos=0.5,sloped, above] {$r$}  (t23);
\draw[->,out=70,in=180](t14) to node[pos=0.5,sloped, above] {$r$} (t24);
\draw[->,out=70,in=180](t22) to node[pos=0.5,sloped, above] {$r$} (t14);
\draw[thin] (3,1) -- (2,-1) -- (4,-1) -- cycle;
\draw[thin] (6,3) -- (5,1) -- (7,1) -- cycle;
\draw[thin] (6,0) -- (5,-2) -- (7,-2) -- cycle;
\end{tikzpicture}%
}
\end{figure}
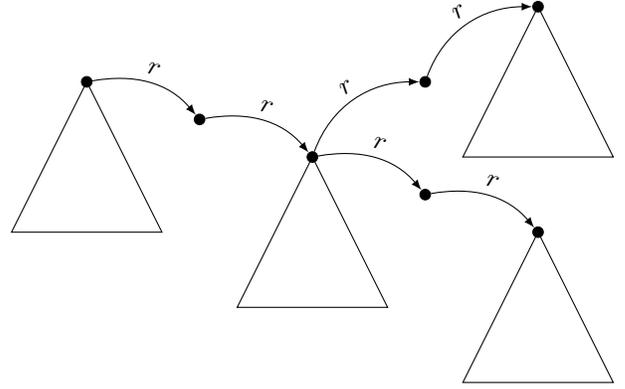
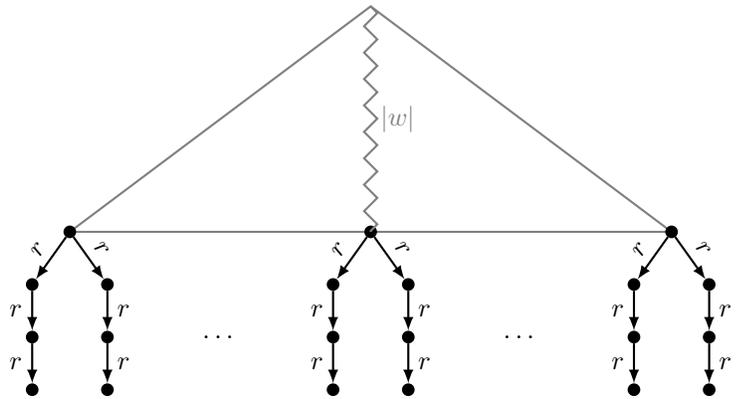
\begin{figure}\caption{Configuration Tree}\label{fig:conf-tree}
\centerline{
\begin{tikzpicture}[>=latex,thick,point/.style={circle,draw=black,minimum size=1.4mm,inner sep=0pt}]
\coordinate (A) at (0,3);
\coordinate (B) at (-4,0);
\coordinate(C) at (4,0);
\draw[color=gray] (A) -- (B) -- (C) -- cycle;
\foreach \x in  {-4, 0, 4}{
\node[point, fill=black] (cp) at (\x,0) {};
\node[point, fill=black] (lp) at (\x-0.5,-0.7) {};
\node[point, fill=black] (llp) at (\x-0.5,-1.4) {};
\node[point, fill=black] (lllp) at (\x-0.5,-2.1) {};
\node[point, fill=black] (rp) at (\x+0.5,-0.7) {};
\node[point, fill=black] (rrp) at (\x+0.5,-1.4) {};
\node[point, fill=black] (rrrp) at (\x+0.5,-2.1) {};
 %
\draw[->] (cp) to  node[pos=0.5,sloped, above] {$r$} (lp);
\draw[->] (cp) to  node[pos=0.5,sloped, above] {$r$} (rp);
\draw[->] (lp) to  node[pos=0.5, left] {$r$} (llp);
\draw[->] (llp) to  node[pos=0.5, left] {$r$} (lllp);
\draw[->] (rp) to  node[pos=0.5, right] {$r$} (rrp);
\draw[->] (rrp) to  node[pos=0.5, right] {$r$} (rrrp);
}
\node at (-2,-1.4){$\dots$}; 
\node at (2,-1.4){$\dots$}; 
\draw[color=gray, zigzag] (A) to node[pos= 0.5, right]{$|w|$} (0,0);
\end{tikzpicture}%
}
\end{figure}
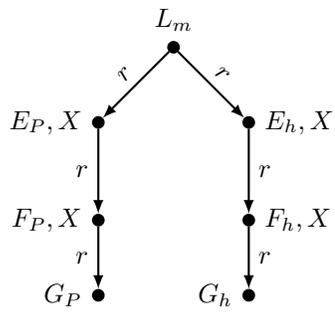
\begin{figure}\caption{Cell gadget}\label{fig:cell-gadget}
\centerline{%
\begin{tikzpicture}[>=latex,thick,point/.style={circle,draw=black,minimum size=1.4mm,inner sep=0pt}]
%
\node[point, fill=black, label=above:{$L_m$}] (cp) at (0,0) {};
\node[point, fill=black, label=left:{$E_P,X$}] (lp) at (-1,-1) {};
\node[point, fill=black, label=left:{$F_P,X$}] (llp) at (-1,-2.3) {};
\node[point, fill=black, label=left:{$G_P$}] (lllp) at (-1,-3.3) {};
\node[point, fill=black, label=right:{$E_h,X$}] (rp) at (1,-1) {};
\node[point, fill=black, label=right:{$F_h,X$}] (rrp) at (1,-2.3) {};
\node[point, fill=black, label=left:{$G_h$}] (rrrp) at (1,-3.3) {};
\draw[->] (cp) to  node[pos=0.5,sloped, above] {$r$} (lp);
\draw[->] (cp) to  node[pos=0.5,sloped, above] {$r$} (rp);
\draw[->] (lp) to  node[pos=0.5, left] {$r$} (llp);
\draw[->] (llp) to  node[pos=0.5, left] {$r$} (lllp);
\draw[->] (rp) to  node[pos=0.5, right] {$r$} (rrp);
\draw[->] (rrp) to  node[pos=0.5, right] {$r$} (rrrp);
\end{tikzpicture}
}
\end{figure}
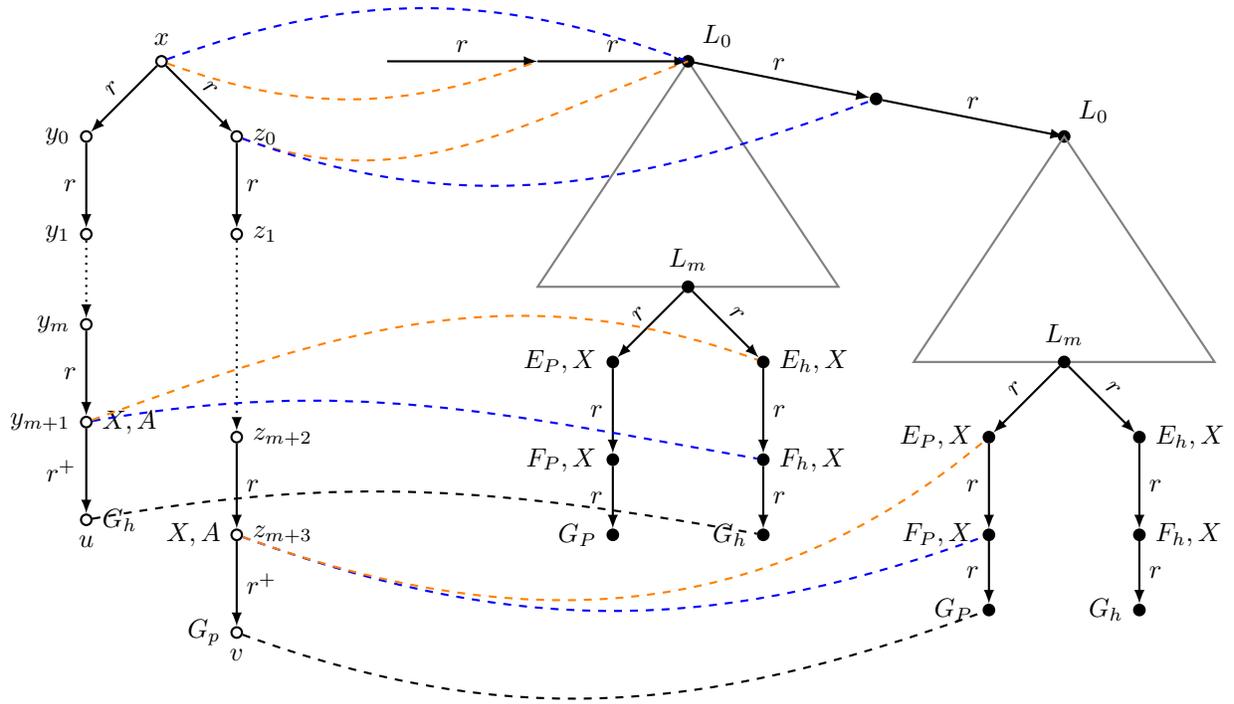
\begin{figure*}\caption{Basic Query}\label{fig:BasicQ}
\centerline{
\begin{tikzpicture}[>=latex,thick,point/.style={circle,draw=black,minimum size=1.4mm,inner sep=0pt}]
\node[point, fill=white, label=above:{$x$}] (x) at (0,0) {};
\node[point,  fill=white, label=left:{$y_0$}] (y0) at (-1,-1) {};
\node[point, fill=white, label=left:{$y_1$}] (y1) at (-1,-2.3) {};
\node[point,  fill=white, label=left:{$y_m$}] (ym) at (-1,-3.5) {};
\node[point,  fill=white, label=left:{$y_{m+1}$}, label= right:{$X, A$}] (X) at (-1,-3.5- 1.3) {};
\node[point,  fill=white, label=below:{$u$}, label= right:{$G_h$}] (Gh) at (-1,-3.5 - 2.6) {};
\node[point,  fill=white, label=right:{$z_0$}] (z0) at (1,-1) {};
\node[point, fill=white, label=right:{$z_1$}] (z1) at (1,-2.3) {};
\node[point,  fill=white, label=right:{$z_{m+2}$}] (zm2) at (1,-5) {};
\node[point,  fill=white, label=right:{$z_{m+3}$}, label=left:{$X, A$}] (XR) at (1,-5-1.3) {};
\node[point,  fill=white, label=below:{$v$}, label=left:{$G_p$}] (Gp) at (1,-5-2.6) {};
\draw[->] (x) to  node[pos=0.5,sloped, above] {$r$} (y0);
\draw[->] (x) to  node[pos=0.5,sloped, above] {$r$} (z0);
\draw[->] (y0) to  node[pos=0.5, left] {$r$} (y1);
\draw[->,dotted](y1) to (ym);
\draw[->] (ym) to  node[pos=0.5, left] {$r$} (X);
\draw[->] (zm2) to  node[pos=0.5, right] {$r$} (XR);
\draw[->] (X) to  node[pos=0.5, left] {$r^+$} (Gh);
\draw[->] (XR) to  node[pos=0.5, right] {$r^+$} (Gp);
\draw[->] (z0) to  node[pos=0.5, right] {$r$} (z1);
\draw[->,dotted] (z1) to  (zm2);
\coordinate (A) at (7,0);
\coordinate (B) at (5,-3);
\coordinate(C) at (9,-3);
\draw[color=gray] (A) -- (B) -- (C) -- cycle;
\node[point, fill=black, label=above right:$L_0$] at (A) {};
\node[point, fill=black, label=above right:$L_0$] at (12,-1) {};
\draw[color=gray] (12,-1) -- (10,-4) -- (14,-4) -- cycle;
\node[point, fill=black] (p) at (9.5, -0.5){};
\draw[->] (A) to node[pos=0.5, above]{$r$} (p);
\draw[->] (5,0) to node[pos=0.5, above]{$r$} (A);
\draw[->] (3,0) to node[pos=0.5, above]{$r$} (5,0);
\draw[->](p) to node[pos=0.5, above]{$r$} (12,-1);
\node[shift={(7,-3)}, point, fill=black, label=above:{$L_m$}] (cp) at (0,0) {};
\node[shift={(7,-3)}, point, fill=black, label=left:{$E_P,X$}] (lp) at (-1,-1) {};
\node[shift={(7,-3)}, point, fill=black, label=left:{$F_P,X$}] (llp) at (-1,-2.3) {};
\node[shift={(7,-3)}, point, fill=black, label=left:{$G_P$}] (lllp) at (-1,-3.3) {};
\node[shift={(7,-3)}, point, fill=black, label=right:{$E_h,X$}] (rp) at (1,-1) {};
\node[shift={(7,-3)}, point, fill=black, label=right:{$F_h,X$}] (rrp) at (1,-2.3) {};
\node[shift={(7,-3)}, point, fill=black, label=left:{$G_h$}] (rrrp) at (1,-3.3) {};
\draw[->] (cp) to  node[pos=0.5,sloped, above] {$r$} (lp);
\draw[->] (cp) to  node[pos=0.5,sloped, above] {$r$} (rp);
\draw[->] (lp) to  node[pos=0.5, left] {$r$} (llp);
\draw[->] (llp) to  node[pos=0.5, left] {$r$} (lllp);
\draw[->] (rp) to  node[pos=0.5, right] {$r$} (rrp);
\draw[->] (rrp) to  node[pos=0.5, right] {$r$} (rrrp);
\node[shift={(12,-4)}, point, fill=black, label=above:{$L_m$}] (Lm) at (0,0) {};
\node[shift={(12,-4)}, point, fill=black, label=left:{$E_P,X$}] (EP) at (-1,-1) {};
\node[shift={(12,-4)}, point, fill=black, label=left:{$F_P,X$}] (FP) at (-1,-2.3) {};
\node[shift={(12,-4)}, point, fill=black, label=left:{$G_P$}] (GP) at (-1,-3.3) {};
\node[shift={(12,-4)}, point, fill=black, label=right:{$E_h,X$}] (EH) at (1,-1) {};
\node[shift={(12,-4)}, point, fill=black, label=right:{$F_h,X$}] (FH) at (1,-2.3) {};
\node[shift={(12,-4)}, point, fill=black, label=left:{$G_h$}] (GH) at (1,-3.3) {};
\draw[->] (Lm) to  node[pos=0.5,sloped, above] {$r$} (EP);
\draw[->] (Lm) to  node[pos=0.5,sloped, above] {$r$} (EH);
\draw[->] (EP) to  node[pos=0.5, left] {$r$} (FP);
\draw[->] (FP) to  node[pos=0.5, left] {$r$} (GP);
\draw[->] (EH) to  node[pos=0.5, right] {$r$} (FH);
\draw[->] (FH) to  node[pos=0.5, right] {$r$} (GH);
\draw[dashed,out=-20,in=200,] (Gp) to (GP);
\draw[dashed,out=10,in=170] (Gh) to (rrrp);
\draw[dashed,out=10,in=170, blue] (X) to (rrp);
\draw[dashed,out=20,in=160, orange] (X) to (rp);
\draw[dashed,out=20,in=160, blue] (x) to (A);
\draw[dashed,out=-20,in=200, orange] (x) to (5,0);
\draw[dashed,out=-20,in=200, orange] (z0) to (A);
\draw[dashed,out=-20,in=200, blue] (z0) to (p);
\draw[dashed,out=-20,in=200, blue] (XR) to (FP);
\draw[dashed,out=-20,in=220, orange] (XR) to (EP);
%
%
\end{tikzpicture}}
\end{figure*}
%
%
